\documentclass[final,11pt]{arxiv}
\usepackage{natbib}

\usepackage[utf8]{inputenc} 
\usepackage[T1]{fontenc}    
\usepackage{hyperref}       
\usepackage{url}            
\usepackage{booktabs}       
\usepackage{amsfonts}       
\usepackage{nicefrac}       
\usepackage{microtype}      
\usepackage{xcolor}         

\usepackage{microtype}
\usepackage{graphicx}
\usepackage{booktabs}
\usepackage{lipsum,wrapfig}

\usepackage{hyperref}

\usepackage{amssymb}
\usepackage{amsmath}
\usepackage{graphicx}
\usepackage{color}
\usepackage{graphics}
\usepackage{natbib}
\usepackage{mathtools}
\usepackage[mathscr]{euscript}
\usepackage{enumitem}
\usepackage{thm-restate}

\usepackage{algorithmic}
\usepackage{algorithm}

\newif\ifnotes
\notestrue

\ifnotes

\newcommand{\rnote}[1]{ [\textcolor{magenta}{Raef: #1}] }

\newcommand{\mnote}[1]{ [\textcolor{red}{Mehryar: #1}] }

\else

\newcommand{\rnote}[1]{}
\newcommand{\mnote}[1]{}
\newcommand{\thnote}[1]{}
\fi

\let\vec\undefined
\usepackage{MnSymbol}
\DeclareMathAlphabet\mathbb{U}{msb}{m}{n}
\usepackage{xpatch}

\def\Rset{\mathbb{R}}

\let\P\undefined

\DeclareMathOperator*{\P}{\mathbb{P}}
\DeclareMathOperator*{\E}{\mathbb E}

\DeclareMathOperator*{\argmin}{argmin}

\DeclarePairedDelimiter{\abs}{\lvert}{\rvert} 
\DeclarePairedDelimiter{\bracket}{[}{]}
\DeclarePairedDelimiter{\curl}{\{}{\}}
\DeclarePairedDelimiter{\norm}{\|}{\|}
\DeclarePairedDelimiter{\paren}{(}{)}
\DeclarePairedDelimiter{\tri}{\langle}{\rangle}

\newcommand{\cL}{\mathcal{L}}
\newcommand{\cM}{\mathcal{M}}
\newcommand{\cN}{\mathcal{N}}
\newcommand{\cO}{\mathcal{O}}

\newcommand{\sD}{{\mathscr D}}

\newcommand{\sG}{{\mathscr G}}
\newcommand{\sH}{{\mathscr H}}

\newcommand{\sQ}{{\mathscr Q}}

\newcommand{\sU}{{\mathscr U}}
\newcommand{\sV}{{\mathscr V}}
\newcommand{\sW}{{\mathscr W}}
\newcommand{\sX}{{\mathscr X}}
\newcommand{\sY}{{\mathscr Y}}

\newcommand{\bz}{{\mathbf z}}

\newcommand{\bg}{{\mathbf g}}

\newcommand{\sfp}{{\mathsf p}}
\newcommand{\sfq}{{\mathsf q}}

\newcommand{\sfu}{{\mathsf u}}
\newcommand{\sfv}{{\mathsf v}}
\newcommand{\sfB}{{\mathsf B}}

\newcommand{\sfF}{{\mathsf F}}

\newcommand{\sfJ}{{\mathsf J}}
\newcommand{\sfK}{{\mathsf K}}

\newcommand{\dis}{\mathrm{dis}}

\newcommand{\proj}{\mathrm{Proj}}

\newcommand{\Rad}{\mathfrak R}
\newcommand{\bsigma}{{\boldsymbol \sigma}}

\newcommand{\h}{\hat}
\newcommand{\wh}{\widehat}
\newcommand{\ov}{\overline}
\newcommand{\wt}{\widetilde}
\newcommand{\e}{\varepsilon}

\newcommand{\pub}{\mathsf{Pub}}
\newcommand{\prv}{\mathsf{Priv}}

\newcommand{\cnvx}{\mathsf{CnvxAdap}}
\newcommand{\pcnvx}{\mathsf{CnvxAdap}}
\newcommand{\ncnvx}{\mathsf{NCnvxAdap}}
\newcommand{\pncnvx}{\mathsf{NCnvxAdap}}

\newcommand{\ex}[2]{\underset{#1}{\mathbb{E}}\left[ #2 \right]}

\newcommand{\lap}{\mathsf{Lap}}
\newcommand{\sql}{\ell_{\mathsf{sq}}}
\newcommand{\dprv}{\h d_{\textsc{DP}}}

\newcommand{\RETURN}{\mbox{{\bf return }}}

\newcommand{\mathset}[2][]{#1 \{ #2 #1 \} }
\newcommand{\ignore}[1]{}

\newcommand{\nlabel}{%
  \addtocounter{ALC@line}{-1} \refstepcounter{ALC@line}
  \label
}

\hypersetup{
  colorlinks   = true,
  urlcolor     = blue,
  linkcolor    = blue,
  citecolor    = blue
}

\usepackage{dsfont}
\usepackage{enumitem}
\usepackage[toc,page,header]{appendix}
\usepackage{times}
\title[Differentially Private Domain Adaptation 
with Theoretical Guarantees]{Differentially Private Domain Adaptation \\
with Theoretical Guarantees}

\coltauthor{%
 \Name{Raef Bassily} \Email{bassily.1@osu.edu}\\
 \addr The Ohio State University and Google Research%
 \AND
 \Name{Corinna Cortes} \Email{corinna@google.com}\\
 \addr Google Research, New York%
 \AND
 \Name{Anqi Mao} \Email{aqmao@cims.nyu.edu}\\
 \addr Courant Institute of Mathematical Sciences, New York%
 \AND
 \Name{Mehryar Mohri} \Email{mohri@google.com}\\
 \addr Google Research and Courant Institute of Mathematical Sciences, New York%
}

\begin{document}

\maketitle

\begin{abstract}%

  In many applications, the labeled data at the learner's disposal is
  subject to privacy constraints and is relatively limited. To derive
  a more accurate predictor for the target domain, it is often
  beneficial to leverage publicly available labeled data from an
  alternative domain, somewhat close to the target domain. This is the
  modern problem of supervised domain adaptation from a public source
  to a private target domain.
  We present two $(\e, \delta)$-differentially private adaptation
  algorithms for supervised adaptation, for which we make use of a
  general optimization problem, recently shown to benefit from
  favorable theoretical learning guarantees. Our first algorithm is
  designed for regression with linear predictors and shown to solve a
  convex optimization problem. Our second algorithm is a more general
  solution for loss functions that may be non-convex but Lipschitz and
  smooth.
  While our main objective is a theoretical analysis, we also report the
results of several experiments first demonstrating that the
non-private versions of our algorithms outperform adaptation baselines
and next showing that, for larger values of the target sample size or
$\e$, the performance of our private algorithms remains close to that
of the non-private formulation.%
  
\end{abstract}

\section{Introduction}
\label{sec:intro}
In many applications, the labeled data at hand is
not sufficient to train an accurate model for a target
domain. Instead, a large amount of labeled data may be available from
another domain, a \emph{source domain}, somewhat close to the target
domain. The problem then consists of using the labeled data available
from both the source and target domains to come up with a more
accurate predictor for the target domain. This is the setting of
\emph{supervised domain adaptation}.

The problem faced in practice is often even more challenging, since
the labeled data from the target domain can be sensitive and subject
to privacy constraints \citep*{BassilyMohriSuresh2020}. For example, a
corporation such as an airline company, or an institution such as a
hospital, may seek to train a classifier based on private labeled data
it has collected, as well as a large amount of data available from a
public domain. To share the classifier internally, let alone share it
publicly to the benefit of other institutions or individuals, it may
have to train the classifier with privacy guarantees. In the absence
of the public domain data and without the adaptation scenario, the
framework of differential privacy
\citep*{DworkMcSherryNissimSmith2006, DworkRoth2014} can be used to
privately learn a classifier that can be shared publicly. But, how can
we rigorously design a differentially private algorithm for supervised
domain adaptation?

This paper deals precisely with the problem of devising a private
(supervised) domain adaptation algorithm with theoretical guarantees
for this scenario. Our scenario \ignore{is motivated by a number of
applications in practice since it} covers any standard privacy learning
scenario where additional data from another public source is sought.

\ignore{To do so, we benefit from a recent theoretical analysis of supervised
domain adaptation.  }The problem of domain adaptation has been
theoretically investigated in a series of publications in the past and
the notion of \emph{discrepancy} was shown to be a key divergence
measure to the analysis of adaptation \citep{KiferBenDavidGehrke2004,
  BlitzerCrammerKuleszaPereiraWortman2008,
  BenDavidBlitzerCrammerKuleszaPereiraVaughan2010,
  MansourMohriRostamizadeh2009, CortesMohri2011, MohriMunozMedina2012,
  CortesMohri2014, CortesMohriMunozMedina2019,
  ZhangLongWangJordan2020}.  Building on this prior work,
\citet{AwasthiCortesMohri2024} recently gave a general theoretical
analysis of supervised adaptation that holds for any method relying on
reweighting the source and target labeled samples, including
reweighting methods that depend on the predictor selected. Thus, the
analysis covers a large number of algorithms in adaptation, including
importance weighting \citep{SugiyamaEtAl2007a,
  LuZhangFangTeshimaSugiyama2021,SugiyamaKrauledatMuller2007,
  CortesMansourMohri2010,ZhangYamaneLuSugiyama2020}, KLIEP
\citep{SugiyamaEtAl2007a}, Kernel Mean Matching (KMM)
\citep{HuangSmolaGrettonBorgwardtScholkopf2006}, discrepancy
minimization (DM) or generalized discrepancy minimization
\citep{CortesMohri2014, CortesMohriMunozMedina2019}. The authors also
suggested a general optimization problem that consists of minimizing
the right-hand side of their learning bound.

\textbf{Contributions.} We present two $(\e, \delta)$-differentially
private adaptation algorithms for supervised adaptation, based on the
optimization problem of \citet{AwasthiCortesMohri2024}.
We first consider a regression setting with linear predictors
(Section~\ref{sec:alg_linear_predictors}).  We show that
after a \ignore{non-trivial }suitable reparameterization of the weights assigned to the sample
losses, the optimization problem for adaptation can be formulated as a
joint convex optimization problem over the choice of the predictor and
that of the reparameterized weights.  We then provide an $(\e,
\delta)$-differentially private adaptation algorithm, $\pcnvx$, using
that convex formulation that can be viewed as a variant of noisy
projected gradient descent. We note that noisy gradient descent is a
general technique that has been well studied in the literature of
private optimization \citep{bassily2014private, abadi2016deep,
  wang2017differentially, bassily2019private}. \ignore{Our main contribution
here is to show that the adaptation problem can be formulated as a
convex optimization problem and to prove key properties that make this
problem privately solvable via a known approach in private convex
optimization.} We prove a formal convergence guarantee for our private
algorithm in terms of $\e$ and $\delta$ and the sizes of the source
and target samples.

We then consider in Section~\ref{sec:non-convex} a more general
setting where the loss function may be a non-convex function of the
parameters and is only assumed to be Lipschitz and smooth. This covers
the familiar case where the logistic loss is applied to the output of
neural networks, that is cross-entropy with softmax. We show that,
remarkably, here too, that reparameterization of the weights combined
with the use of the softmax can help us design an $(\e,
\delta)$-differentially private algorithm, $\pncnvx$, that benefits
from favorable convergence guarantees to stationary points of the
objective based on the gradient mapping criterion
\citep{beck2017first}.

While the main objective of our work is a theoretical analysis, we
also report extensive empirical evaluations. In
Section~\ref{sec:experiments}, we present the results of extensive
several conducted for both our convex and non-convex private algorithms.
We demonstrate that the non-private version of our convex algorithm,
$\pcnvx$ ($\epsilon = +\infty$), surpasses existing adaptation methods
for regression, including the state-of-the-art DM algorithm
\cite{CortesMohri2014} and that our private algorithm performs
comparably to its non-private counterpart, showcasing the
effectiveness of our privacy-preserving approach.
Similarly, for our non-convex algorithm, $\pncnvx$, the non-private
version exhibits superior performance compared to baselines in
classification tasks and the performance of our private algorithm
approaches that of its non-private version as the target sample size
and privacy budget ($\e$) increase\ignore{, demonstrating the trade-off between privacy and accuracy}.

\textbf{Related work.} \ignore{The problem of} Private \emph{density
estimation} using a small amount of public data has been studied in
several recent publications, in particular for learning Gaussian
distributions or mixtures of Gaussians, under some assumptions about
the public data
\cite{BieKamathSinghal2022,BenDavidBieCanonneKamathSinghal2023} (see
also \citep{TramerKamathCarlini2023}). The objective is distinct from
our goal of private adaptation in supervised learning.

The most closely related work to ours is the
recent study of \citet{BassilyMohriSuresh2020}, which considers a
similar adaptation scenario with a public source domain and a private
target domain and which also gives private algorithms with theoretical
guarantees.
However, that work can be distinguished from ours in several aspects.
First, the authors consider a purely unsupervised adaptation scenario
where no labeled sample is available from the target domain, while we
consider a supervised scenario. Our study and algorithms can be
extended to the unsupervised or weakly supervised setting using the
notion of \emph{unlabeled discrepancy}
\citep{MansourMohriRostamizadeh2009}, by leveraging upper bounds
on \emph{labeled discrepancy} in terms of unlabeled discrepancy
as in \citep{AwasthiCortesMohri2024}.
Second, the learning guarantees of our private algorithms benefit from
the recent optimization of \citet{AwasthiCortesMohri2024}, which they
show have stronger learning guarantees than those of the DM solution of
\citet{CortesMohri2014} adopted by \citet{BassilyMohriSuresh2020}.
Similarly, in our experiments, our convex optimization solution
outperforms the DM algorithm. Note that the empirical study of
\citet{BassilyMohriSuresh2020} is limited to a single 
artificial dataset, while we present empirical results with several
non-artificial datasets.
Third, our private adaptation algorithms cover regression and
classification, while those of \citet{BassilyMohriSuresh2020} only
address regression with the squared loss.
In Appendix~\ref{app:relatedwork}, we further discuss related work in
adaptation and privacy.

We first introduce in Section~\ref{sec:preliminaries} several basic
concepts and notation for adaptation and privacy, as well as the
learning problem we consider. Next, in Section~\ref{sec:opt} we
describe supervised adaptation optimization by
\citet{AwasthiCortesMohri2024} and derive private versions for that
setting in Section~\ref{sec:alg_linear_predictors} and
Section~\ref{sec:non-convex}. Finally, Section~\ref{sec:experiments}
provides experimental results.  \ignore{Next, in
  Section~\ref{sec:opt}, we describe the optimization problem for
  supervised adaptation introduced by \citet{AwasthiCortesMohri2024},
  which is based on the right-hand side of a generalization bound. As
  a result, the solution benefits from favorable learning
  guarantees. We give a more concise and self-contained proof of that
  generalization bound in Appendix~\ref{app:learningbound}.}

\section{Preliminaries}
\label{sec:preliminaries}

We write $\sX$ to denote the input space and $\sY$ the output space
which, in the regression setting, is assumed to be a measurable subset
of $\Rset$. We will consider a hypothesis set $\sH$ of functions
mapping from $\sX$ to $\sY$ and a loss function $\ell\colon \sY \times
\sY \to \Rset_+$. We will denote by $B > 0$ an upper bound on the loss
$\ell(h(x), y)$ for $h \in \sH$ and $(x, y) \in \sX \times \sY$.
Given a distribution $\sD$ over $\sX \times \sY$, we denote by
$\cL(\sD, h)$ the expected loss of $h$ over $\sD$, 
$\cL(\sD, h) = \E_{(x, y) \sim \sD} \bracket{\ell(h(x), y)}$.

\textbf{Domain adaptation}. We study a (supervised) domain adaptation
problem with a public source domain defined by a distribution
$\sD^\pub$ over $\sX \times \sY$ and a private target domain defined
by a distribution $\sD^\prv$ over $\sX \times \sY$.  We assume that
the learner receives a labeled sample $S^\pub$ of size $m$ drawn
i.i.d.\ from $\sD^\pub$, $S^\pub = \paren*{(x^\pub_1, y^\pub_1),
  \ldots, (x^\pub_m, y^\pub_m)}$, as well as a labeled sample
$\sD^\prv$ of size $n$ drawn i.i.d.\ from $\sD^\prv$, $S^\prv =
\paren*{(x^\prv_1, y^\prv_1), \ldots, (x^\prv_n, y^\prv_n)}$. The size
of the target sample $n$ is typically more modest than that of the
source sample in applications, $n \ll m$, but we will not require that
assumption and will also consider alternative scenarios.
For convenience, we also write $S = \paren*{S^\pub, S^\prv}$ to denote
the full sample of size $m + n$.

\textbf{Learning scenario}. The learning problem consists of using
both samples to select a predictor $h \in \sH$ with {\em privacy
  guarantees} and small expected loss with respect to the target
distribution $\sD^\prv$. The notion of privacy we adopt is that of
\emph{differential privacy}
\citep{dwork2006our,DworkMcSherryNissimSmith2006,DR14}, which in this
context can be defined as follows: given $\varepsilon$ and $\delta >
0$, a (randomized) algorithm $\cM \colon (\sX \times \sY)^{m + n} \to
\sH$ is said to be $(\varepsilon, \delta)$-differentially private if
for any public sample $S^\pub$, for any pair of private datasets
$S^\prv$ and $\hat{S}^\prv \in \paren*{\sX \times \sY}^n$ that differ
in exactly one entry, and for any measurable subset $\cO \subseteq
\sH$, we have: $\P \paren*{\cM((S^\pub, S^\prv)) \in \cO} \leq
e^\varepsilon \, \P \paren*{\cM((S^\pub, \hat{S}^\prv)) \in \cO} +
\delta$. Thus, the information gained by an observer is approximately
invariant to the presence or absence of a sample point in the private
sample.

\textbf{Discrepancy}.\ignore{ Domain adaptation is a challenging learning
problem, even in the non-private setting.} For adaptation to be
successful, the source and target distributions must be close
according to an appropriate divergence measure.
Several notions of \emph{discrepancy} have been shown to be adequate
divergence measures in previous theoretical analyses of adaptation
problems \citep{KiferBenDavidGehrke2004,MansourMohriRostamizadeh2009,
  MohriMunozMedina2012,CortesMohri2014, CortesMohriMunozMedina2019}.
We will denote by $\dis(\sD^\prv, \sD^\pub)$ the \emph{labeled
discrepancy} of $\sD^\prv$ and $\sD^\pub$, also called
$\sY$-discrepancy in
\citep{MohriMunozMedina2012,CortesMohriMunozMedina2019} and defined
by:
\begin{equation}
\label{eq:disc-definition}
  \dis(\sD^\prv, \sD^\pub)
   = 
  \sup_{h \in \sH} \abs*{\cL(\sD^\prv, h) - \cL(\sD^\pub, h)}.
\end{equation}
Labeled discrepancy can be straightforwardly upper bounded in terms of
the $\norm{\, \cdot \,}_1$ distance of the private and public
distributions:
$\dis(\sD^\prv, \sD^\pub) \leq B \norm{\sD^\prv - \sD^\pub}_1$. Some of its
key benefits are that, unlike the
$\norm{\, \cdot \,}_1$-distance, it takes into account the loss
function and the hypothesis set and it can be estimated from finite
samples, also in the privacy preserving setting, see
Section~\ref{sec:alg_linear_predictors}.  Note that, while we are
using absolute values for the difference of expectations, our analysis does not require that and the proofs
hold with a one-sided definition. In some instances, a finer
and more favorable notion of \emph{local discrepancy} can be used, where the supremum is
restricted to a subset $\sH_1 \subset \sH$ \citep{CortesMohriMunozMedina2019,DeMathelinMougeotVayatis2021,ZhangLiuLongJordan2019,
  ZhangLongWangJordan2020}.

\section{Optimization problem for supervised adaptation}
\label{sec:opt}
Let $\h d$ denote the empirical estimate of the discrepancy
based on the samples $S^\pub$ and $S^\prv$:
\begin{align}
\label{eq:dis_estimate}
\h d &
\triangleq \sup_{h \in \sH}
\abs*{\frac{1}{n}\sum_{i = m + 1}^{m + n}\ell(h(x_i), y_i)
  - \frac{1}{m}\sum_{i = 1}^m \ell(h(x_i), y_i)}.
\end{align}
Let $\sfq \in [0, 1]^{m + n}$ denote a vector of weights over the full
sample $(S^\pub, S^\prv)$, which, depending on their values, are used
to emphasize or deemphasize the loss on each sample $(x_i, y_i)$.  We
also denote by $\ov \sfq^\pub$ the \emph{total weight} on the first
$m$ (public) points, $\ov \sfq^\pub = \sum_{i = 1}^m \sfq_i$, and by
$\ov \sfq^\prv$ the \emph{total weight} on the next $n$ (private)
ones, $\ov \sfq^\prv = \sum_{i = m + 1}^{m + n} \sfq_i$. Note that
$\sfq$ is not required to be a distribution.
Then, the following joint optimization problem based on a
$\sfq$-weighted empirical loss and the empirical discrepancy $\h d$
was suggested by \citet{AwasthiCortesMohri2024} for supervised domain
adaptation.
\ifdim\columnwidth=\textwidth
\begin{align}
  \min_{\substack{h \in \sH\\ \sfq \in \sQ}}
  & \ \sum_{i = 1}^{m + n}\sfq_i\bracket*{\ell(h(x_i),y_i) + \h d \, 1_{i \leq m}}
  + \lambda_1\norm{\sfq - \sfp^{0}}_1 + \lambda_2\norm{\sfq}_2
  + \lambda_\infty\norm{\sfq}_\infty, \label{opt-origin}
\end{align}
\else
\begin{align}
  \min_{\substack{h \in \sH\\ \sfq \in \sQ}}
  & \sum_{i = 1}^{m + n}
  \sfq_i\bracket*{\ell(h(x_i),y_i) + \h d 1_{i \leq m}}
  \label{opt-origin}\\
  & + \lambda_1 
  \norm{\sfq - \sfp^{0}}_1 + \lambda_2\norm{\sfq}_2 +
  \lambda_\infty\norm{\sfq}_\infty, \nonumber
\end{align}
\fi where $\sQ = [0, 1]^{m + n}$ and where $\lambda_1$, $\lambda_2$
and $\lambda_\infty$ are non-negative hyperparameters. Here, $\sfp^0$
is a \emph{reference} or \emph{ideal} reweighting choice, further
discussed below.

This optimization is directly based on minimizing the right-hand side
of a generalization bound (see Theorem~\ref{th:learningbound},
Appendix~\ref{app:learningbound}), for which we give a self-contained
and more concise proof. Moreover, \cite{AwasthiCortesMohri2024}
established a corresponding lower bound for any reweighting technique
in terms of the $\sfq$-weighted empirical loss and the discrepancy of
the distributions $\sD^\prv$ and $\sD^\pub$ for a weight vector
$\sfq$. This further validates the significance of the generalization
bound. It suggests that the optimization problem \eqref{opt-origin}
admits the strongest theoretical learning guarantee we can hope for
and can be regarded as an \emph{ideal} algorithm among
reweighting-based algorithms for supervised adaptation.

The key idea behind the optimization is to assign different weights
$\sfq$ to labeled samples to account for the presence of distinct
domain distributions, akin to reweighting strategies such as
importance weighting.
The success of
adaptation hinges on a favorable balance of several crucial factors
expressed by Theorem~\ref{th:learningbound}. We aim to select a predictor $h$ with a small
$\sfq$-weighted empirical loss ($\sum_i \sfq_i \ell(h(x_i),
y_i)$). Yet, we must limit the total $\sfq$-weight assigned to source
domain samples if the empirical discrepancy $\h d$ is substantial
(captured by the $\h d$ term in \eqref{opt-origin}). Avoiding disproportionate weighting on
a few points is critical to maintaining an adequate effective sample
size, addressed by the inclusion of the norm-2 term
$\norm{\sfq}_2$. The norm-1 term encourages the choice of $\sfq$ not
deviating significantly from the reference weights $\sfp^0$, while the
norm-$\infty$ term relates to controlling complexity.
A careful empirical tuning of the hyperparameters helps achieve a
judicious balance between these terms, leading to a well-performing
adaptation process. A more detailed discussion is given in Appendix~\ref{app:add}.

A natural reference $\sfp^0$, which we assume in the
following, is an $\alpha$-mixture of the empirical distributions $\h
\sD^\pub$ and $\h \sD^\prv$ associated to the samples $S^\pub$ and
$S^\prv$: $\sfp^0 = \alpha \h \sD^\pub + (1 - \alpha) \h \sD^\prv $,
with $\alpha \in (0, 1)$.  Thus, $\sfp^0_i$ is equal to
$\frac{\alpha}{m}$ if $i \in [1, m]$, $\frac{1 - \alpha}{n}$
otherwise. The mixture parameter $\alpha$ can be chosen as a function
of the estimated discrepancy\ignore{ between $\sD^\prv$ and $\sD^\pub$}.

An important advantage of the solution based on this optimization
problem is that the weights $\sfq$ are selected in conjunction with
the predictor $h$. This is unlike most reweighting techniques in the
literature, such as importance weighting \citep{SugiyamaEtAl2007a,
  LuZhangFangTeshimaSugiyama2021,SugiyamaKrauledatMuller2007,
  CortesMansourMohri2010,ZhangYamaneLuSugiyama2020}, KLIEP
\citep{SugiyamaEtAl2007a}, Kernel Mean Matching (KMM)
\citep{HuangSmolaGrettonBorgwardtScholkopf2006}, discrepancy
minimization (DM) \citep{CortesMohri2014}, and gapBoost
\citep{WangMendezCaiEaton2019}, that consist of first pre-determining
some weights $\sfq$ irrespective of the choice of $h$, and
subsequently selecting $h$ by minimizing a $\sfq$-weighted empirical
loss.

\textbf{Optimality and theoretical guarantees.}  In light of the
theoretical properties already underscored, we define an ideal
algorithm for \emph{private} supervised adaptation as one that
achieves $(\epsilon, \delta)$-DP and returns a solution closely
approximating that of problem \eqref{opt-origin}.  This paper
introduces two differentially private algorithms that precisely
fulfill these criteria. We also present empirical evidence
demonstrating the proximity of the performance of our private
algorithms to that of the non-private optimization problem
\eqref{opt-origin}.

For a fixed choice of $\sfq$, an additional error term of
$\Omega\paren[\big]{\frac{\sqrt{d\log \frac{1}{\delta}}}{\e n}}$ is
necessary to ensure privacy for convex ERM
\citep{bassily2014private,steinke2015between}. Here, $d$ represents
the dimension of the parameter space. Therefore, the term \ignore{
$\frac{\sqrt{d\log \frac{1}{\delta}}}{\e n}$} is necessary in the
expected loss of any $(\epsilon, \delta)$-differentially private
supervised adaptation algorithm based on sample reweighting.
Remarkably, the theoretical guarantee that we prove for our first
algorithm only differs from that of \eqref{opt-origin} by a term
closely matching $\frac{\sqrt{d\log \frac{1}{\delta}}}{\e n}$.

\section{Private adaptation algorithm for regression with
  linear predictors}
\label{sec:alg_linear_predictors}

In this section, we consider a regression problem with the squared
loss and using linear predictors, for which we give a differentially
private adaptation algorithm. 

We consider an input space
$\sX = \curl*{x \in \Rset^d \colon \norm{x}_2 \leq r}$, $r > 0$, an
output space $\sY = \curl*{y \in \Rset \colon \abs{y} \leq 1}$, and a
family of bounded linear predictors
$\sH = \curl*{x \mapsto w \cdot x \mid w \in \sW}$, where
$\sW = \curl*{w \colon \norm{w}_2 \leq \Lambda}$, for some
$\Lambda > 0$. This covers scenarios where we fix lower layers of
a pre-trained neural network and only seek to learn the parameters
of the top layer. 

Note that the squared loss is bounded for $x \in \sX$ and $w \in \sW$:
$\sql(w \cdot x, y) = (w \cdot x - y)^2 \leq (\Lambda r + 1)^2
\triangleq B$.  It is also $G$-Lipschitz, $G \triangleq 2r(\Lambda r +
1)$, with respect to $w \in \sW$ since $\abs*{\sql(w\! \cdot\! x, y) -
  \sql(w'\! \cdot\! x, y)} \leq G \norm{w - w'}_2$ for all $w, w' \in \sW$
and $(x ,y) \in \sX \times \sY$. Furthermore,
$\sql(\cdot, y)$ is convex in  $y \in \sY$.

To devise an $(\epsilon, \delta)$-differentially private algorithm for
our adaptation scenario, we first show how to privately estimate the
discrepancy term. Next, we show that the natural optimization problem
\eqref{opt-origin} \ignore{described in the previous section }can be
cast as a convex optimization problem, for which we design a favorable
private solution.

\textbf{Discrepancy estimates}.
Since $\sql$ is convex with respect to its first argument, problem
\eqref{eq:dis_estimate} can be cast as two difference-of-convex problems
((DC)-programming) by removing the absolute value and
considering both possible signs.
Each of these problems can be solved using the DCA algorithm
of \cite{TaoAn1998} (see also
\citep{YuilleRangarajan2003,SriperumbudurTorresLanckriet2007}). Furthermore,
for the squared loss $\sql$ with our linear hypotheses, the DCA
method is guaranteed to reach a global optimum \citep{TaoAn1998}.

Now, observe that the empirical discrepancy $\h d$ defined in
(\ref{eq:dis_estimate}) is estimated using private data, therefore,
the addition of noise is crucial to ensure the differential privacy
(DP). This term solely depends on the dataset $S$ and not on the
particular choice of $h$ or $\sfq$ and it remains unchanged through
the iterations of Algorithm~\ref{Alg:ngd}.  Therefore, its (private)
estimation can be performed upfront, prior to the algorithm's
execution.
Furthermore, the sensitivity of $\h d$ with respect to replacing one
point in $S^\prv$ is at most $B/n$. Thus, we can first generate an
$\epsilon/2$-differentially private version $\dprv =
\mathsf{Proj}_{[0, B]}\paren*{\h d + \nu}$ of $\h d$ by augmenting $\h
d$ with $\nu \sim \lap(2B/(\epsilon n))$, where $\lap(2B/(\epsilon
n))$ is a Laplace distribution with scale $2B/(\epsilon n)$, and
projecting over the interval $[0, B]$. Thus, we modify the objective
function \eqref{opt-origin} by replacing $\h d$ with $\dprv$.  Note
also that by the properties of the Laplace distribution and the fact
that $\h d \geq 0$, the expected excess optimization error due to this
modification is in $O\paren*{B/(\epsilon n)}$, which, as we shall see,
is dominated by the excess error due to privately optimizing the new
objective. In the following, in the private optimization algorithm we
present (Algorithm~\ref{Alg:ngd}), we instantiate the privacy
parameter with $\e/2$ so that the overall algorithm is $(\e,
\delta)$-differentially private.  Alternatively, we could add noise
directly to the gradients to account for the empirical discrepancy
term estimated from the private data. However, a straightforward
analysis shows that this approach would introduce significantly more
noise to the gradients.

After substituting with our choice of a uniform reference distribution, as mentioned in
Section~\ref{sec:preliminaries}, the problem reduces to privately
solving the following optimization problem:
\begin{align}
  & \min_{\substack{\norm{w}_2 \leq \Lambda\\ \sfq \in \sQ}} 
  \label{opt-private-main} 
  \sum_{i = 1}^{m + n} \sfq_i \bracket*{\paren*{w \cdot x_i - y_i}^2
    + \dprv 1_{i \leq m}} \\
  & + \mspace{-5mu} \lambda_1 \bracket*{\sum_{i = 1}^m \abs*{\sfq_i - \frac{\alpha}{m}}
    + \mspace{-10mu}  \sum_{i = m + 1}^{n + m} \abs*{\sfq_i - \frac{1 - \alpha}{n}}} 
  + \lambda_2\norm{\sfq}_2 + \lambda_\infty\norm{\sfq}_\infty. \nonumber
\end{align}
This optimization presents two main challenges: (1) while it is convex
with respect to $w$ and with respect to $\sfq$, it is not jointly
convex; (2) the gradient sensitivity with respect to $\sfq$ of the
objective is a constant and thus not favorable to derive differential
privacy guarantees. In the following, we will show how both issues can
be tackled.
Inspecting \eqref{opt-private-main} leads to the
following useful observation.

Observe that for each $i\in [m]$, the objective is increasing in
$\sfq_i$ over $\bracket*{\frac{\alpha}{m}, 1}$ and similarly, for
each $i \in [m + 1, m + n]$, it is increasing in $\sfq_i$ over the
interval $\bracket*{\frac{1 - \alpha}{n}, 1}$. Thus, the following
stricter constraints on $\sfq$ in \eqref{opt-private-main} do not
affect the optimal solution: $\forall i\in [m], \ 0 \leq \sfq_i\leq
\frac{\alpha}{m}; \forall i\in [n], \ 0 \leq \sfq_{m + i} \leq \frac{1
  - \alpha}{n}$.  The problem can thus be equivalently formulated as:
\begin{align}
  \min_{\norm{w}_2 \leq \Lambda , \sfq}
  & \ \sum_{i = 1}^{m + n}\sfq_i\bracket*{\paren*{w\cdot x_i - y_i}^2
    + \dprv  1_{i \leq m}} 
   + \lambda_1 \bracket[\Big]{1 - \sum_{i = 1}^{m + n}\sfq_i}
  + \lambda_2\norm{\sfq}_2 + \lambda_\infty\norm{\sfq}_\infty \label{opt-private-simplified}\\
  \text{s.t. }
  & \ \forall i \in [m], \ 0 \leq \sfq_i\leq \frac{\alpha}{m};
  \forall i \in [n], \ 0 \leq \sfq_{n + i} \leq \frac{1 - \alpha}{n}.\nonumber
\end{align}

\textbf{Convex-optimization formulation}. We now derive a convex
formulation of this optimization problem, which enables us to devise
an efficient private algorithm with formal convergence guarantee. We
introduce new variables $\sfu_i = \frac{1}{\sfq_i}$,
$\forall i \in [1, m + n]$, and use the following upper
bound that holds by the convexity of $x \mapsto 1/x$ on $\Rset_+^*$:
\begin{align*}
  & 1 - \sum_{i = 1}^{m + n}\frac{1}{\sfu_i}
  \leq \paren*{\frac{\alpha}{m}}^2 \sum_{i = 1}^m \sfu_i
  + \paren*{\frac{1 - \alpha}{n}}^2\sum_{i = m + 1}^{m + n}\sfu_i - 1.
\end{align*}
\ignore{which holds for $\sfu_i \geq \frac{m}{\alpha}$ for $i \in [1, m]$ and
$\sfu_i \geq \frac{n}{1 - \alpha}$ for $i\in [m + 1, m + n]$. }This
yields the following optimization problem in $(w, \sfu)$:
\begin{align}
  \label{opt-private-convex}
  & \min_{\norm{w}_2 \leq \Lambda, \sfu} \ \sum_{i = 1}^{m + n} \frac{\paren*{w\cdot x_i - y_i}^2 + \dprv \, 1_{i \leq m}}{\sfu_i}\\
&  \mspace{-40mu}  + \kappa_1\bracket*{\frac{\alpha^2}{m^2} \mspace{-5mu} \sum_{i = 1}^m \mspace{-2mu} \sfu_i \mspace{-2mu} + \mspace{-2mu}  \frac{(1 - \alpha)^2}{n^2} \mspace{-7mu} \sum_{i = m + 1}^{m + n} \mspace{-10mu} \sfu_i \mspace{-2mu} - \mspace{-2mu} 1}
  \mspace{-2mu} + \mspace{-2mu} \kappa_2\bracket*{\sum_{i = 1}^{m + n} \mspace{-5mu} \frac{1}{\sfu_i^2}}^{\frac{1}{2}}
  \mspace{-14mu} + \mspace{-2mu}  \frac{\kappa_\infty}{\displaystyle \min_{i}\sfu_i} \nonumber\\
  \text{s.t. }
  & \ \forall i \in [m], \sfu_i \geq \frac{m}{\alpha}; \
  \forall i \in [n], \sfu_{m + i} \geq \frac{n}{1 - \alpha},\nonumber
\end{align}
with new
hyperparameters $\kappa_1, \kappa_2, \kappa_\infty$. The problem \eqref{opt-private-convex} is a joint
convex optimization problem in $w$ and $\sfu$ since the constraints on
$u$ are affine, the constraint on $w$ is convex, and since each term
$\frac{\paren*{w \cdot x_i - y_i}^2}{\sfu_i}$ is jointly convex in
$(w, \sfu_i)$ as a quadratic-over-linear function
\citep{BoydVandenberghe2014}. We will denote by $\sfF(w, \sfu)$ the
objective function and by $\sU$ the feasible set of $\sfu$, $\sU =
\paren*{[\frac{m}{\alpha}, \infty)^m \times [\frac{n}{1 - \alpha},
      \infty)^n}$.

In the following, we will also use the shorthands $\sfu^\pub =
(\sfu_1, \ldots, \sfu_m)$ and $\sfu^\prv = (\sfu_{m + 1}, \ldots,
\sfu_{m + n})$, and denote by $\wh L^\pub(w, \sfu^\pub)$ and $\wh
L^\prv(w, \sfu^\prv)$ their  contributions to the empirical
loss:
  $\wh L^\pub(w, \sfu^\pub)
  \triangleq \sum_{i = 1}^m  \frac{(w\cdot x^\pub_i - y^\pub_i)^2 + \dprv}{\sfu^\pub_i}$,
  $\wh L^\prv(w, \sfu^\prv)
   \triangleq \sum_{i = 1}^n  \frac{(w\cdot x^\prv_i - y^\prv_i)^2}{\sfu^\prv_i}$.
Note that, since differential privacy is closed under postprocessing,
$\dprv$ is safe to publish as a differentially private estimate of $\h
d$. Thus, the only term in $\sfF(w, \sfu)$ that is sensitive from a privacy
perspective is $\wh L^\prv(w, \sfu^\prv)$.

Let $\nabla_{w}, \nabla_{\sfu^\pub}$, and $\nabla_{\sfu^\prv}$ denote
the gradients with respect to $w, \sfu^\pub$, and $\sfu^\prv$,
respectively. We will also use the shorthand $\bar B\triangleq B +
\kappa_1 + \kappa_2 + \kappa_\infty$. The following lemma shows several
important properties of the objective function. The proof is given in
Appendix~\ref{app:lemma-convex}.

\begin{algorithm}[t]
    \caption{ $\pcnvx$ Private adaptation algorithm based on $\sfF$}
    \begin{algorithmic}[1]
      \REQUIRE
      $S^\pub \in \paren*{\sX \times \sY}^m$;
      $S^\prv \in \paren*{\sX \times \sY}^n$;
      privacy parameters $(\varepsilon, \delta)$;
      hyperparameters $\kappa_1, \kappa_2, \kappa_\infty$;
      number of iterations $T$.
      \STATE Choose $(w_0, \sfu_0)$ in $\sW \times \sU$ arbitrarily.
          \STATE Set $\sigma_1: =  \frac{2s_1 \sqrt{T\log(\frac{3}{\delta})}}{{\varepsilon}},$ where $s_1: =  \frac{2(1 - \alpha)G}{n}$.\nlabel{step:GD-sigma-1}
          \STATE Set $\sigma_2 : =  \frac{2s_2 \sqrt{T\log(\frac{3}{\delta})}}{{\varepsilon}},$ where $s_2: =  \frac{(1 - \alpha)^2B}{n^2}$. \nlabel{step:GD-sigma-2}
          \STATE Set step sizes $\eta_w := \frac{\Lambda}{\sqrt{T \,\paren*{G^2+d\sigma_1^2}}}$, $\eta_{\sfu^\pub}: = \frac{m^{3/2}}{\sqrt{T} \,\alpha^2(B + \bar B)}$,~ and  $\eta_{\sfu^\prv}: = \frac{n^{3/2}}{\sqrt{T \,\paren*{(1 - \alpha)^4 \bar B^2+n^4\sigma_2^2}}}$.\nlabel{step:step-sizes}
          \FOR{$t = 0$ to $T-1$}
          \STATE $w_{t+1}: =  w_t - \eta_w  \paren*{\nabla_{w}\sfF(w_t, \sfu_t) + \bz_t}$, where $\bz_t\sim \cN(\mathbf{0}, \sigma_1^2\mathbb{I}_d)$. \label{step:noise_add_w}
          \STATE If $\norm{w_{t+1}}_2> \Lambda$ then $w_{t+1} \gets \Lambda \frac{w_{t+1}}{\norm{w_{t+1}}_2}$.\nlabel{step:project-w-convex}
          \STATE $\sfu^\pub_{t+1}: =  \sfu^\pub_t - \eta_{\sfu^\pub} \, \nabla_{\sfu^\pub}\sfF(w_t, \sfu_t)$.
          \STATE For every $i\in [m]$, set $\sfu^\pub_{i, t+1}\gets \max\paren*{\sfu^\pub_{i, t+1}, \frac{m}{\alpha}}$. \nlabel{step:project-u_pub-convex}
          \STATE $\sfu^\prv_{t+1}: = \sfu^\prv_t-\eta_{\sfu^{\prv}} \, \paren*{\nabla_{\sfu^\prv}\sfF(w_t, \sfu_t) + \bz'_t}$, where $\bz'_t\sim \cN(\mathbf{0}, \sigma_2^2\mathbb{I}_n)$. \label{step:noise_add_u_prv} 
          \STATE For every $i\in [n]$, set $\sfu^\prv_{i, t+1}\gets \max\paren*{\sfu^\prv_{i, t+1}, \frac{n}{1 - \alpha}}$. \nlabel{step:project-u_priv-convex}
          \ENDFOR
\STATE \textbf{return} $\paren*{\bar{w}, \bar{\sfu}}  = \frac{1}{T}\sum_{t = 1}^T\paren*{w_t, \sfu_t}$.
    \end{algorithmic}
    \label{Alg:ngd}
\end{algorithm}

\begin{restatable}{lemma}{PropertiesOfConvexObjective}
\label{lem:_properties}
The following properties hold for the objective function $\sfF$.

\noindent (i) The following upper bounds hold for the gradients, for all $(w,
  \sfu)\in \sW \times \sU$:
     $\norm{\nabla_w \sfF(w, \sfu)}_2
     \leq G$,
     $\norm{\nabla_{\sfu^\pub} \sfF(w, \sfu)}_2
     \leq \frac{\alpha^2 \paren*{B + \bar B}}{m^{3/2}}$, and
     $\norm{\nabla_{\sfu^\prv} \sfF(w, \sfu)}_2
     \leq \frac{(1 - \alpha)^2 \bar B}{n^{3/2}}$.

\noindent (ii) The $\ell_2$-sensitivity of $\nabla_{w} \wh L^\prv(w,
  \sfu^\prv)$ with respect to changing one private data point is at
  most $\frac{2(1 - \alpha)G}{n}$;\\
\noindent (iii) The $\ell_2$-sensitivity of $\nabla_{\sfu^\prv} \wh L^\prv(w,
  \sfu^\prv)$ with respect to changing one private data point is at
  most $\frac{(1 - \alpha)^2B}{n^2}$.
      
\end{restatable}

Algorithm~\ref{Alg:ngd}, denoted $\pcnvx$, gives pseudocode for
our differentially private adaptation algorithm based on the convex
problem \eqref{opt-private-convex}. Our algorithm is a variant of
noisy projected gradient descent with
steps~\ref{step:project-w-convex}, \ref{step:project-u_pub-convex},
and \ref{step:project-u_priv-convex}\ignore{in
  Algorithm~\ref{Alg:ngd}} implementing the Euclidean projection of
$(w_t, \sfu_t)$ onto the constraint set $\sW \times \sU$. The
 non-private version of Algorithm~\ref{Alg:ngd} we will
denote $\cnvx_\infty$ for $\epsilon\!=\!\infty.$
The following provides both a DP and convergence
guarantee for our algorithm.

\begin{restatable}{theorem}{ConvexAlgorithmGuarantees}
\label{th:guarantees-convex-priv-alg}
Algorithm~\ref{Alg:ngd} is $(\varepsilon, \delta)$-differentially
private. Furthermore, let $\paren*{w^\ast, \sfu^\ast}$ be a minimizer
of $\sfF(w, \sfu)$, then, the expected optimization error of the
solution $(\bar{w}, \bar{\sfu})$ returned by Algorithm~\ref{Alg:ngd}
is bounded as follows: 
\begin{align*}
 \sfF(\bar{w}, \bar{\sfu}) - \sfF(w^\ast, \sfu^\ast)
   \leq O\paren[\bigg]{
    \frac{G\Lambda\sqrt{d \log \frac{1}{\delta}}}{n\varepsilon}
    + \frac{B \max\paren*{1, \norm{\sfu_0 - \sfu^\ast}_2^2}
      \sqrt{\log \frac{1}{\delta}}}{n^{3/2}\varepsilon}},
\end{align*}
for $T \geq \max\paren*{1, \frac{n^2\varepsilon^2}{d(1 -
    \alpha)^2\log(\frac{1}{\delta})}, \frac{\bar B^2\varepsilon^2}{B^2
    \log(\frac{1}{\delta})}, \frac{\varepsilon^2\bar B^2
    n^3}{\log(\frac{1}{\delta})B^2 m^3}}$.
\end{restatable}
The more explicit form of the bound as well as the proof are presented
in Appendix~\ref{app:thm-convex}.
We note here that, for sufficiently large $n$, our bound scales as
$O\paren*{{G\Lambda\sqrt{d\log(1/\delta)}}/{n\varepsilon}}$. This
bound on the optimization error is essentially optimal for our
optimization problem under differential privacy. To see this, note
that our optimization task is generally harder than the standard
empirical risk minimization (ERM) as the latter task can be viewed as
a simple instantiation of our optimization problem, where the optimal
weights $\sfq$ are uniform on the private data and zero on the public
data, and they are given to the algorithm beforehand (hence, the
optimization algorithm is only required to optimize over the
parameters vector $w$). Hence, the known lower bound of
$\Omega\paren*{{G\Lambda\sqrt{d\log(1/\delta)}}/{n\varepsilon}}$
on private convex ERM \citep{bassily2014private}\footnote{The lower
bound in \citep{bassily2014private} does not have the
$\sqrt{\log(1/\delta)}$ factor, but it's been known that the lower
bound can be improved to include this factor via the more recent
results of \citep{steinke2015between}.}  implies a lower bound on the
optimization error in our problem. Note that this implies that the
additional error incurred by our algorithm due to privacy (i.e.,
compared to the best non-private algorithm for
optimizing~\eqref{opt-origin}) scales as
$O\paren*{{\sqrt{d\log(1/\delta)}}/{\e n}}$, which matches the
necessary additional error (due to privacy) in the expected loss of
any private algorithm based on sample reweighting, as discussed in
Section~\ref{sec:opt}. Formally stated: 

\begin{restatable}{theorem}{ConvexAlgorithmErrorOptimality}
\label{th:convex-optimal-error}
Suppose $n\geq \frac{\paren*{B \max\paren*{1, U^\pub + U^\prv}}^2}{(G\Lambda)^2 d}$ and $T \geq \max\paren*{1, \frac{n^2\varepsilon^2}{d(1 -
    \alpha)^2\log(\frac{1}{\delta})}, \frac{\bar B^2\varepsilon^2}{B^2
    \log(\frac{1}{\delta})}, \frac{\varepsilon^2\bar B^2
    n^3}{\log(\frac{1}{\delta})B^2 m^3}}$ in Algorithm~\ref{Alg:ngd}. Then, the resulting expected optimization error is $O\paren*{\frac{G\Lambda\sqrt{d\log(1/\delta)}}{n\varepsilon}}$, which is optimal.
\end{restatable}

\section{Private adaptation -- more general settings}
\label{sec:non-convex}

Here, we consider a general possibly non-convex setting, where the loss function is only
assumed to be $G$-Lipschitz and $\beta$-smooth, that is
differentiable, with $\beta$-Lipschitz gradient in the parameter $w$
with respect to the $\ell_2$-norm. To
simplify, we will here abusively adopt
the notation $\ell(w, x, y)$ to denote the loss associated with a
parameter vector $w \in \sW$ (defining a hypothesis) and a labeled
point $(x, y) \in \sX\times\sY$. Since the $\ell_2$-diameter of $\sW$
is $\Lambda$-bounded and the loss is Lipschitz, we assume without loss
of generality that $\ell$ is uniformly bounded by $B = G\Lambda$ over 
$\sW\times \sX \times \sY$.

Our goal is to privately optimize a more general version of problem~\ref{opt-private-simplified}, where the squared loss is replaced with any such loss $\ell$. This is in general a
non-convex optimization problem and
finding a global minimum is hence
intractable. An alternative,  widely adopted  in
the literature, is to find a stationary
point of the non-convex objective.

\textbf{Challenges for the design of a private adaptation algorithm.}
Before we discuss the stationarity criterion, we first describe our
approach. One issue with the objective when expressed as a function of
$\sfq = (\sfq^\pub, \sfq^\prv)$ is that its gradient with respect to
$\sfq^\prv$ admits an $\Omega(1)$ sensitivity. That can improved by
introducing new variables $\tilde \sfq^\pub =
\frac{\alpha}{m}\sfq^\pub$ and $\tilde{\sfq}^\prv = \frac{1 -
  \alpha}{n}\sfq^\prv$, thereby reducing the sensitivity of the
gradient with respect to $\tilde \sfq^\prv$ to
$O\paren*{\frac{1}{n}}$. Since this gradient is $n$-dimensional, the
magnitude of the noise added for privacy will be
$O\paren*{\frac{1}{\sqrt{n}}}$. However,  we can
further enhance the sensitivity if we resort to the reparameterization
technique described in Section~\ref{sec:alg_linear_predictors}. As we
show in the sequel, by applying the transformation of variables
$\sfu_i = \frac{1}{\sfq_i}$, $i \in [m + n],$ we are able to reduce
the sensitivity of the gradient components, and hence achieve better
convergence guarantees.

Another challenge we face here is the non-smoothness of the
objective. Together with the non-convex functions, it
leads to weaker convergence guarantees even in the non-private setting
\citep{arjevani2019lower, shamir2020can}. Although the aforementioned
transformation ensures that the term corresponding to $\norm{\sfq}_2$
is smooth as a function of $\sfu$, the term corresponding to
$\norm{\sfq}_\infty$, i.e., $\frac{1}{\min_{i\in [m + n]}\sfu_i}$,
remains non-smooth. To address this issue, we replace that term with
its $\mu$-softmax approximation, namely, $\frac{1}{\mu}
\log\paren*{\sum_{i = 1}^{m + n}e^{\mu/\sfu_i}}$, where $\mu >0$ is
the softmax approximation parameter. A basic fact about this
$\mu$-softmax approximation is that its approximation error is
uniformly bounded by $O\paren*{\frac{1}{\mu} \log(m + n)}$. With $\mu = O\paren*{\sqrt{m + n}}$, the excess error we incur \ignore{ due
to optimizing the objective with the softmax approximation} is
$\widetilde{O}\paren*{1/\sqrt{m + n}}$.

Thus, our goal is to privately find a stationary point of the
following optimization problem:
\begin{align}
  \min_{\substack{\norm{w}_2
      \leq \Lambda\\ \sfu}} & \ \sum_{i = 1}^{m + n} \frac{\ell(w, x_i, y_i) + \dprv \, 1_{i \leq m}}{\sfu_i} 
  \label{opt-private-non-convex}  
     + \lambda_1 \bracket*{1 - \sum_{i = 1}^{m + n}\frac{1}{\sfu_i}} 
  + \lambda_2 \bracket*{\sum_{i = 1}^{m + n}\frac{1}{\sfu_i^2}}^{\frac{1}{2}}
  \mspace{-10mu}
   + \frac{\lambda_\infty}{\mu} \log \bracket*{\sum_{i = 1}^{m + n}e^{\frac{\mu}{\sfu_i}} \mspace{-5mu}}  \\
  \text{s.t. }
  & \ \forall i \in [m],  \sfu_i\geq \frac{m}{\alpha}; \
  \forall i \in [n], \sfu_{m + i} \geq \frac{n}{1 - \alpha}. \nonumber
\end{align}
We denote the objective in
(\ref{opt-private-non-convex}) by $\sfJ(w, \sfu)$, and let
$\sV\triangleq \sW\times \sU$, where $\sW\times \sU$ is the constraint
set in the above problem.

In the following lemma, we show several useful properties of $\sfJ$
that will be crucial for proving the privacy and convergence guarantees
of our private algorithm. We stress that our reparameterization idea,
which led to the above problem formulation, was key
to ensuring such properties. That and the following lemma are crucial results
enabling us to devise a private
adaptation solution with guarantees for this general non-convex setting.
The proof is given in Appendix~\ref{app:lemma-nonconvex}.

\begin{restatable}{lemma}{PropertiesOfNonConvexObjective}
\label{lem:properties_nonconvex}
The objective function $\sfJ$ admits the following properties:\\
\noindent (i) $\sfJ$ satisfies the same properties as those stated for $\sfF$
  in Lemma~\ref{lem:_properties}.\\
\noindent (ii) Assume $m^{\frac{1}{3}} = O(n), n = O(m^3)$, and $\mu =
  O\paren[\big]{(m + n)^{\frac{2}{3}}}$. Then, $\sfJ$ is
  $\bar{\beta}$-smooth over $\sV$, where $\bar{\beta}=O(\beta)$.
\end{restatable}

\textbf{Gradient mapping as a stationarity criterion.}  Here, we adopt
a standard stationarity criterion given by the norm of the
\emph{gradient mapping} \citep{beck2017first}. The $\gamma$-gradient
mapping of a function $f\colon \sV\rightarrow \Rset$ at $\sfv\in \sV$
is denoted and defined by $\sG_{f, \gamma}(\sfv) \triangleq
\gamma\paren[\big]{\sfv - \proj_{\sV}\paren[\big]{\sfv-\frac{1}{\gamma}\nabla
    f(\sfv)}}$.
A point $\sfv^\ast\in \sV$ is a stationary point of $f$ if and only if
$\norm{\sG_{f, \gamma}(\sfv^\ast)}_2 = 0$. For an $r$-smooth function
$f$, $\norm{\sG_{f, r}(\cdot)}_2$ serves as a measure of convergence
to a stationary point \citep{beck2017first, ghadimi2016mini,
  li2018simple}.

\textbf{Private algorithm for general adaptation settings:} Our
private algorithm for the non-convex problem, denoted by $\pncnvx$, is
a variant of noisy projected gradient descent. Our algorithm
admits exactly the same generic description as
Algorithm~\ref{Alg:ngd}, including the settings of the noise variances
$\sigma_1^2$ and $\sigma_2^2$, with the following differences: (1) we
use $\sfJ$ instead of $\sfF$; (2) we set
$\eta_w = \eta_{\sfu^\pub} = \eta_{\sfu^\prv} = \eta \triangleq
\frac{1}{\bar{\beta}}$, where $\bar{\beta}$ is the smoothness
parameter given in Lemma~\ref{lem:properties_nonconvex}; (3) the
algorithm returns $(w_{t^\ast}, \sfu_{t^\ast})$, where $t^\ast$ is
drawn uniformly from $[T]$. The full pseudocode
of $\pncnvx$ is given in Appendix~\ref{app:alg-nonconvex}. We will denote
by $\ncnvx_\infty$ the
 non-private version of $\pncnvx$.

Our algorithm is similar to that of
\citet{wang2019differentially}, but our analysis is distinct and
yields stronger guarantee than their Theorem~2. The
convergence measure in that reference is based on the norm of the
noisy projected gradient. In contrast, our convergence guarantee is
based on the noiseless gradient mapping $\sG_{\sfJ, \bar{\beta}}(w_t,
\sfu_t)$, which we view as more relevant.
The following theorem gives privacy and convergence guarantees for our
algorithm, see proof in Appendix~\ref{app:thm-nonconvex}; it establishes the total number of iterations
(gradient updates) of our algorithm is in $O\paren*{\varepsilon/
  \sqrt{d \log (1/\delta)}}$.

\begin{restatable}{theorem}{NonConvexAlgorithmGuarantees}
\label{thm:nonconvex}
Algorithm $\pncnvx$ is $(\varepsilon,
\delta)$-differentially private. Moreover, for the choice $T =
O\paren*{{\varepsilon n}/{\sqrt{d \log(1/\delta)}}},$ the
output of the algorithm satisfies the following bound on the norm of
the gradient mapping: $\norm{\sG_{\sfJ, \bar{\beta}}(w_{t^\ast},
  \sfu_{t^\ast})}_2^2 = O
\paren[\Big]{{\sqrt{\bar{\beta}d\log(1/\delta)}}/{\varepsilon
    n}}$.
\end{restatable}

\section{Experiments}
\label{sec:experiments}

We here report the results of several experiments for both
our convex private adaptation algorithm for regression, $\pcnvx$\ introduced in Section~\ref{sec:alg_linear_predictors}, 
and our non-convex private algorithm for more general settings, $\pncnvx$ introduced in Section~\ref{sec:non-convex}. Since state-of-the-art domain adaptation algorithms in general have not been extended to the private supervised domain adaptation we consider, we adopt the following experimental validation procedure. We first compare our non-private versions of the algorithms with the other domain adaptation algorithms. Here we demonstrate that our proposed algorithms outperforms state-of-the-art baselines. Our further experiments then serve to demonstrate that our private versions in performance remain close to the non-private versions.  

\ignore{
We first compare
our non-private algorithm $\cnvx_\infty$ with several baselines
for domain adaptation, including the DM algorithm\ignore{ (discrepancy
minimization)} \citep{CortesMohri2014} that was previously reported as
the state-of-the-art for adaptation for regression tasks. Next, we
compare our private adaptation algorithm $\pcnvx$\ with the
non-private solution of $\cnvx_\infty$. In the general setting, we compare
logistic regression classifiers trained by our non-private algorithm
$\ncnvx_\infty$ with several baselines and then demonstrate that the
performance of our private adaptation algorithm $\pncnvx$ remains
close to that of $\ncnvx_\infty$.
}

\ignore{
The \texttt{{traffic}} dataset from the Minnesota Department of
Transportation \citep{TaekMuKwon2004,Dua:2019} where the goal is to
predict the traffic volume. We create source and target by splitting
based on the time of the day. This leads to $2200$ examples from the
source and $1000$ examples from the test set ($200$ for training,
$400$ for validation and $400$ for testing).
}


\textbf{Non-private comparison to baselines. Regression.}  We consider five
regression datasets with dimensions as high as $384$ from the UCI
machine learning repository \citep{Dua:2019}, the \texttt{{Wind}},
\texttt{{Airline}}, \texttt{{Gas}},  \texttt{{News}} and \texttt{{Slice}}. Detailed
information about the sample sizes of these datasets is given in
Table~\ref{tbl:dataset} (Appendix~\ref{app:add-convex}).
We compare our non-private convex algorithm $\cnvx_\infty$ with the
Kernel Mean Matching (KMM) algorithm
\citep{HuangSmolaGrettonBorgwardtScholkopf2006} and the DM algorithm
\citep{CortesMohri2014}.  For all algorithms, we do model selection on
the target validation set and report in
Table~\ref{tbl:regression-real} the mean and standard deviation on the
test set over $10$ random splits of the target training and validation
sets.  We report relative MeanSquaredErrors, MSE, normalized so that
training on the target data only has an MSE of 1.0. Thus, MSE numbers
less than one signify performance improvements achieved by leveraging
the source data in the algorithm. The results show that our convex
algorithm consistently outperforms the baselines.

\begin{table}[h]
\vskip -.2in
\caption{MSE of non-private convex algorithm against baselines. We
  report relative errors normalized so that training on target only
  has an MSE of $1.0$. The best results are indicated in boldface.}
\begin{center}
\begin{tabular}{@{\hspace{0cm}}llll@{\hspace{0cm}}}
\toprule
Dataset & {KMM} & {DM} & $\cnvx_\infty$\\
\toprule
\small{\tt{Wind}} & $1.009 \pm 0.035$ & $1.009 \pm 0.045$ &  $\mathbf{0.985 \pm 0.019}$ \\
\small{\tt{Airline}} & $2.716 \pm 0.202$ & $1.547 \pm 0.068$  & $\mathbf{0.992} \pm \mathbf{0.012}$ \\
\small{\tt{Gas}} & $0.441 \pm 0.034$ & $0.381 \pm 0.028$  &  $\mathbf{0.342 \pm 0.023}$\\
\small{\tt{News}} & $1.162 \pm 0.044$ & $1.006 \pm 0.009$  & $\mathbf{0.995 \pm 0.019}$\\
\small{\tt{Slice}} & $1.282 \pm 0.076$ & $1.218 \pm 0.130$  & $\mathbf{0.992 \pm 0.050}$\\
\bottomrule
\end{tabular}
\end{center}
\vskip -.1in
\label{tbl:regression-real}
\end{table}

\textbf{Comparison of private and non-private algorithm. Regression.}
Having demonstrated the quality of our non-private algorithm $\cnvx_\infty$,
we study the performance and convergence properties of  $\pcnvx$ as a function of the privacy guarantee
$\epsilon$ and training sample size $n$. To obtain larger training set
sizes, we sample the data with replacement (see Appendix~\ref{app:add} for a more detailed discussion). For all experiments, the
privacy parameter $\delta$ is set to be
$0.01$. 

\begin{figure}[t]
\begin{center}
\begin{tabular}{@{}cc@{}}
\includegraphics[scale=0.48]{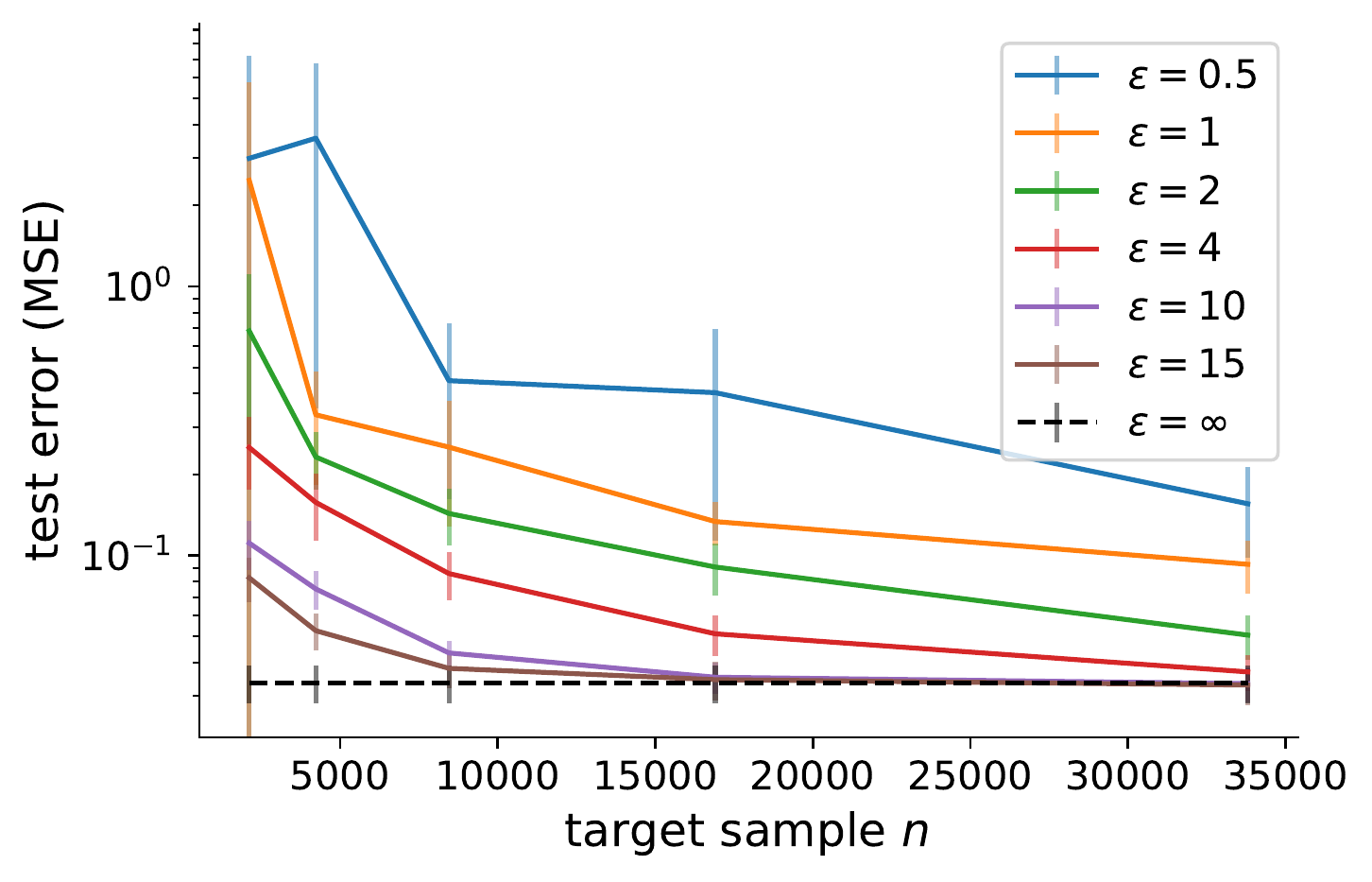}& 
\includegraphics[scale=0.48]{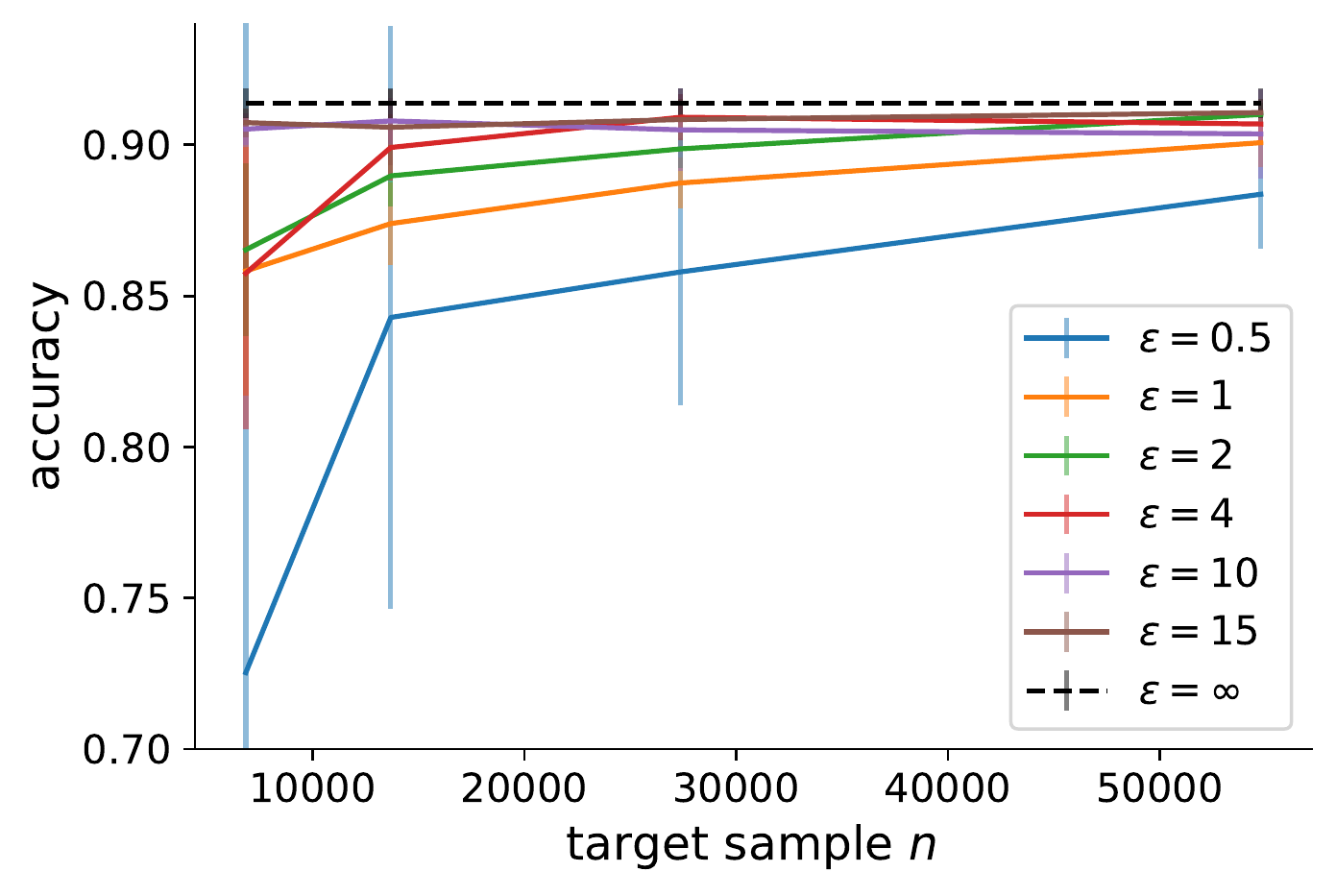}\\[-0.2cm]
{ $\pcnvx$, (MSE): \texttt{Gas}} & { $\pncnvx$, (Accuracy): \texttt{Adult}} \\
\includegraphics[scale=0.48]{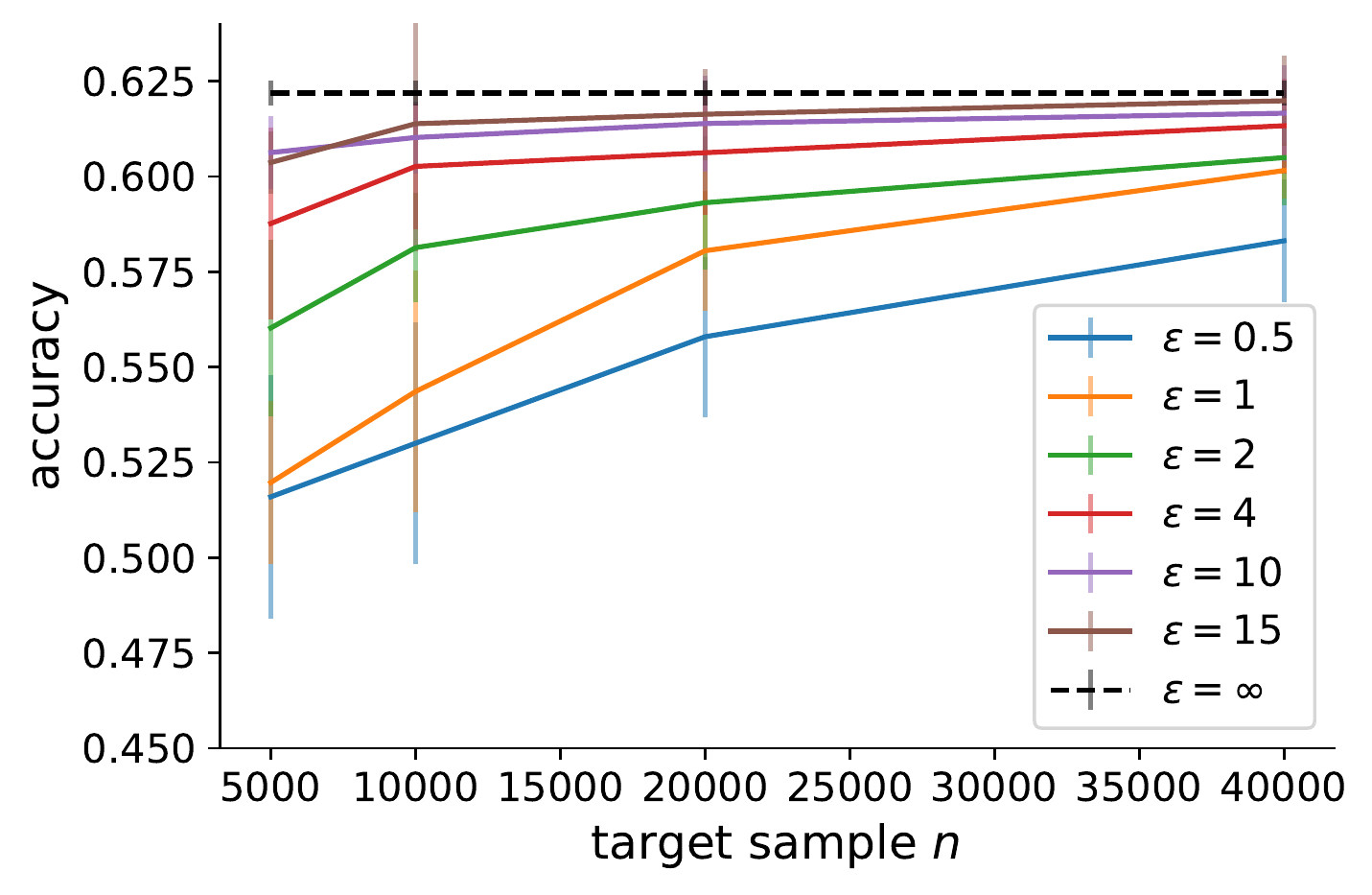}&
\includegraphics[scale=0.48]{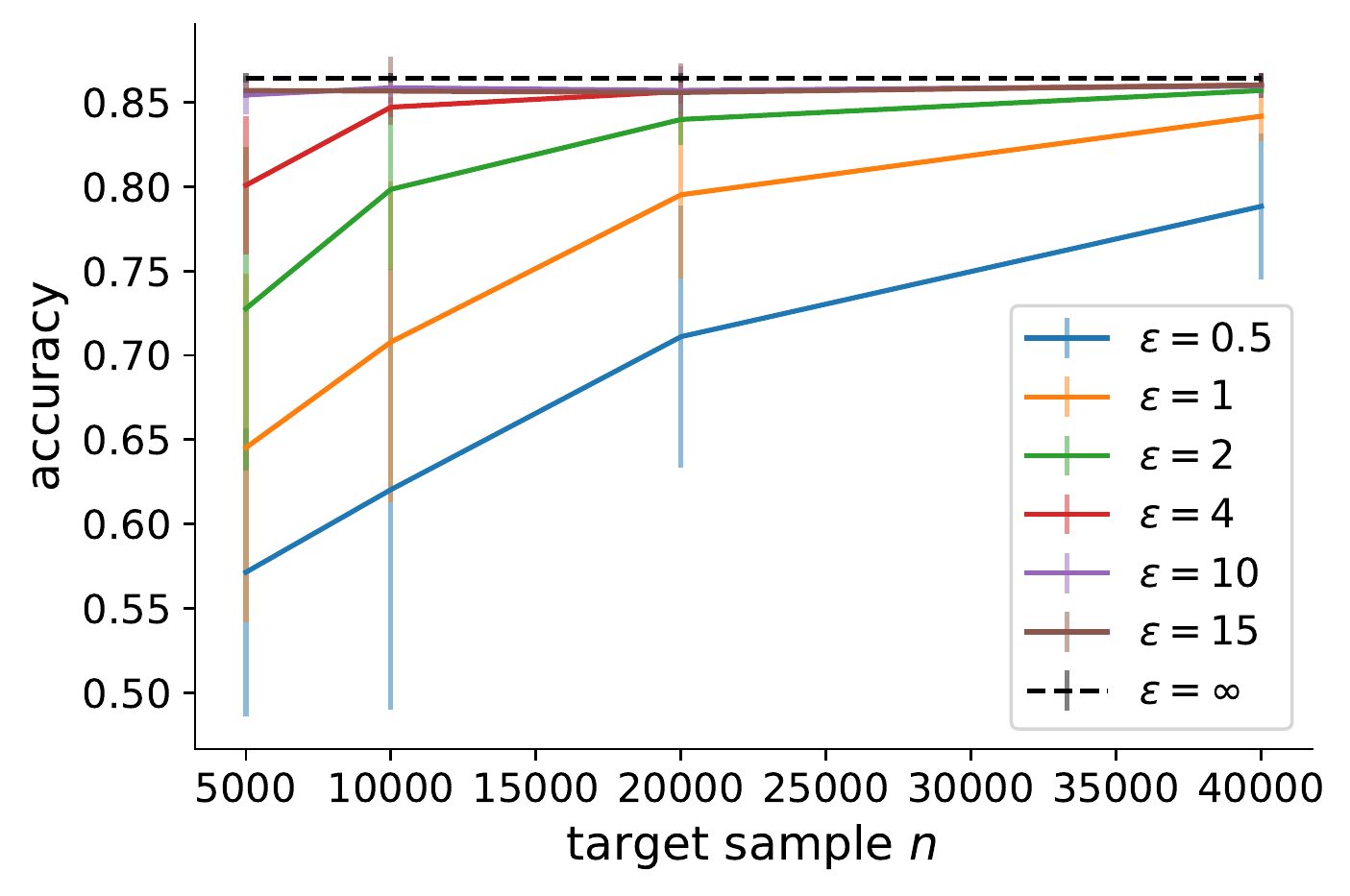}\\[-0.2cm]
{ $\pncnvx$, (Accuracy): \texttt{CIFAR-100}} & { $\pncnvx$, (Accuracy): \texttt{CIFAR-10}} 
\end{tabular}
\vskip -0.1in
\caption{Performance of $\pcnvx$ and $\pncnvx$ on various dataset against number of target samples for various values of $\epsilon$.
\ignore{; Upper right panel: accuracy on \texttt{{Adult}} against 
  target sample size for various values of $\epsilon$; Lower left panel: accuracy on \texttt{{CIFAR-100}} against 
  target sample size for various values of $\epsilon$; 
  Lower right panel: accuracy on \texttt{{CIFAR-10}} against 
  target sample size for various values of $\epsilon$.}}
\label{fig:private}
\end{center}
\vskip -0.4in
\end{figure} 
Figure~\ref{fig:private}(upper left panel) presents the MSE on the
\texttt{{Gas}} dataset over ten runs against the number of target
samples with various values of $\e$ as small as $0.5$. The dashed line
corresponds to the non-private convex solution of $\cnvx_\infty$. We observe
that the performance of our private adaptation algorithm $\pcnvx$
approaches that of the non-private solution ($\e = \infty$), as $\e$
increases and as $n$ increases, thereby verifying our convergence and
theoretical analyses.  For $\e = 10$ or $\e = 15$, the performance of
$\pcnvx$ is close to that of $\cnvx_\infty$ for $n = 10\mathord,000$. For
reference, we also compared  $\cnvx$ with the noisy minibatch SGD
from \citep{bassily2019private}, see Appendix~\ref{app:add},
Figure~\ref{fig:compare}. We verify that our algorithm outperforms minibatch SGD, which only
benefits from the target labeled data.
 \ignore{Our experiments show that for low values of $\e$ and even
   high values of $n$, the private-DM does not out-perform our
   algorithm.}

\ignore{To further illustrate the effectiveness of our
 algorithms} In Table~\ref{tbl-sec:sentiment} we report 
 a series of results
 comparing our private adaptation algorithm $\pcnvx$ to the
 non-private baseline, the DM algorithm \citep{CortesMohri2014}, on
 the multi-domain sentiment analysis dataset
 \citep{blitzer2007biographies}. For $\epsilon\!=\!4$ the performance of $\pcnvx$ is on par with the DM algorithm, and for as low a value as $\epsilon\!=\!10$, it clearly outperforms DM. See Appendix~\ref{app:add} for more details.
 \begin{table}[t]
 \vskip -.2in
  \caption{MSE of $\pcnvx$ for $\epsilon\!=\!4$ and $\epsilon\!=\!10$ against the non-private DM algorithm on the sentiment analysis dataset.}
\begin{center}

\begin{tabular}{@{\hspace{0cm}}llccc@{\hspace{0cm}}}
\toprule
Source & Target &  $\pcnvx$  &  $\pcnvx$ &  DM\\[-0.1cm]
 & & ($\e = 4$) & ($\e = 10$) & \\
\hline
                     & \small{\tt{BOOKS}}& $1.649 \pm 1.909$ & $0.713 \pm 0.625$ &  $0.778 \pm 0.459$ \\
\small{\tt{KITCHEN}} & \small{\tt{DVD}}  & $0.832 \pm 0.640$ & $0.521 \pm 0.354$ &  $1.396 \pm 1.145$ \\
                     & \small{\tt{ELEC}} & $0.790 \pm 0.153$ & $0.555 \pm 0.139$ &  $1.740 \pm 1.483$ \\
\hline
                     & \small{\tt{DVD}}  & $0.929 \pm 0.365$ & $0.637 \pm 0.297$ &  $0.761 \pm 0.562$ \\
\small{\tt{BOOKS}} & \small{\tt{ELEC}}   & $0.708 \pm 0.092$ & $0.485 \pm 0.147$ &  $0.711 \pm 0.320$ \\
                  & \small{\tt{KITCHEN}} & $0.796 \pm 0.126$ & $0.632 \pm 0.181$ &  $0.956 \pm 0.442$ \\
\hline
                     & \small{\tt{ELEC}} & $0.671 \pm 0.120$ & $0.429 \pm 0.189$ &  $0.677 \pm 0.231$ \\
\small{\tt{DVD}} & \small{\tt{KITCHEN}}  & $0.665 \pm 0.197$ & $0.453 \pm 0.189$ &  $1.185 \pm 0.753$ \\
                     & \small{\tt{BOOKS}}& $0.727 \pm 0.288$ & $0.430 \pm 0.143$ &  $0.676 \pm 0.812$ \\
\hline
                   & \small{\tt{KITCHEN}}& $0.671 \pm 0.161$ & $0.439 \pm 0.141$ &  $1.389 \pm 0.755$ \\
\small{\tt{ELEC}} & \small{\tt{BOOKS}}   & $0.743 \pm 0.365$ & $0.351 \pm 0.204$ &  $0.884 \pm 0.639$ \\
                     & \small{\tt{DVD}}  & $0.755 \pm 0.585$ & $0.589 \pm 0.495$ &  $0.843 \pm 0.403$ \\
\bottomrule
\end{tabular}
\end{center}
\vskip -.3in
\label{tbl-sec:sentiment}
\end{table}
Finally, despite the difference in
 scenario, as discussed in Section~\ref{sec:intro}, in Appendix~\ref{app:add}
 we also compare our algorithm to the private DM
 \citep{BassilyMohriSuresh2020} and show that for low
 values of $\e$ and even high values of $n$, the private-DM does not
 outperform our algorithm.

\begin{table}[t]
\vskip -.1in
\caption{Accuracy of non-private non-convex algorithm against baselines. Best results are boldfaced.}
\begin{center}
\vskip -.0in
\begin{tabular}{@{\hspace{0cm}}lcccc@{\hspace{0cm}}}
\toprule
Dataset & Train source & Train target & {KMM} & $\ncnvx_\infty$\\
\toprule
\small{\tt{Adult}}  & $88.44 \pm 0.61$ & $91.09 \pm 0.54$ & $74.62 \pm 0.64$ &  $\mathbf{91.38 \pm 0.49}$ \\
\small{\tt{German}} & $74.16 \pm 1.64$ & $76.40 \pm 1.17$ & $74.43 \pm 1.65$ &  $\mathbf{77.96 \pm 1.31}$ \\
\small{\tt{Accent}} & $51.60 \pm 1.09$ & $92.27 \pm 1.10$ & $52.18 \pm 2.27$ &  $\mathbf{93.95 \pm 0.78}$ \\
\small{\tt{CIFAR-100}} & $58.50 \pm 0.30$ & $59.51 \pm 0.38$ & $58.27 \pm 0.28$ &  $\mathbf{62.19 \pm 0.32}$ \\
\small{\tt{CIFAR-10}} & $83.18 \pm 0.47$ & $85.18 \pm 0.24$ & $83.57 \pm 0.45$ &  $\mathbf{86.44 \pm 0.29}$ \\
\small{\tt{SVHN}} & $84.16 \pm 1.21$ & $86.28 \pm 0.44$ & $84.13 \pm 1.21$ &  $\mathbf{87.34 \pm 0.46}$ \\
\small{\tt{ImageNet}} & $71.20 \pm 1.75$ & $85.70 \pm 0.29$ & $71.25 \pm 1.77$ &  $\mathbf{86.84 \pm 0.38}$ \\
\bottomrule
\end{tabular}
\end{center}
\vskip -.2in
\label{tbl:classification}
\end{table}


\textbf{Non-private comparison to baselines. Classification.} For the general non-convex setting, we experiment with logistic regression classifiers and consider seven datasets. Three
 datasets are from the UCI machine learning repository
\citep{Dua:2019}, the \texttt{{Adult}}, \texttt{{German}} and
\texttt{{Accent}} (see Appendix~\ref{app:add} for the detailed
information about these datasets).
We also convert the \texttt{{CIFAR-100}}, \texttt{{CIFAR-10}} \citep{Krizhevsky09learningmultiple}, \texttt{{SVHN}} \citep{Netzer2011} and \texttt{{ImageNet}} \citep{deng2009imagenet} datasets into domain adaptation tasks by establishing two distinct sampling methods for selecting the source and target data (see Appendix~\ref{app:add-nconvex} for more details).
We first compare our non-private non-convex algorithm $\ncnvx_\infty$ with
the KMM algorithm \citep{HuangSmolaGrettonBorgwardtScholkopf2006}.  We
report in Table~\ref{tbl:classification} the mean and standard
deviation of the accuracy on the test set over $10$ random splits of
the target training (70\%), validation (20\%) and test sets (10\%). The results show that
our non-convex algorithm consistently outperforms the baseline.

\ignore{
\begin{figure}[t]
\vskip -.1in
\begin{center}
\hspace{-.25cm}
\includegraphics[scale=0.35]{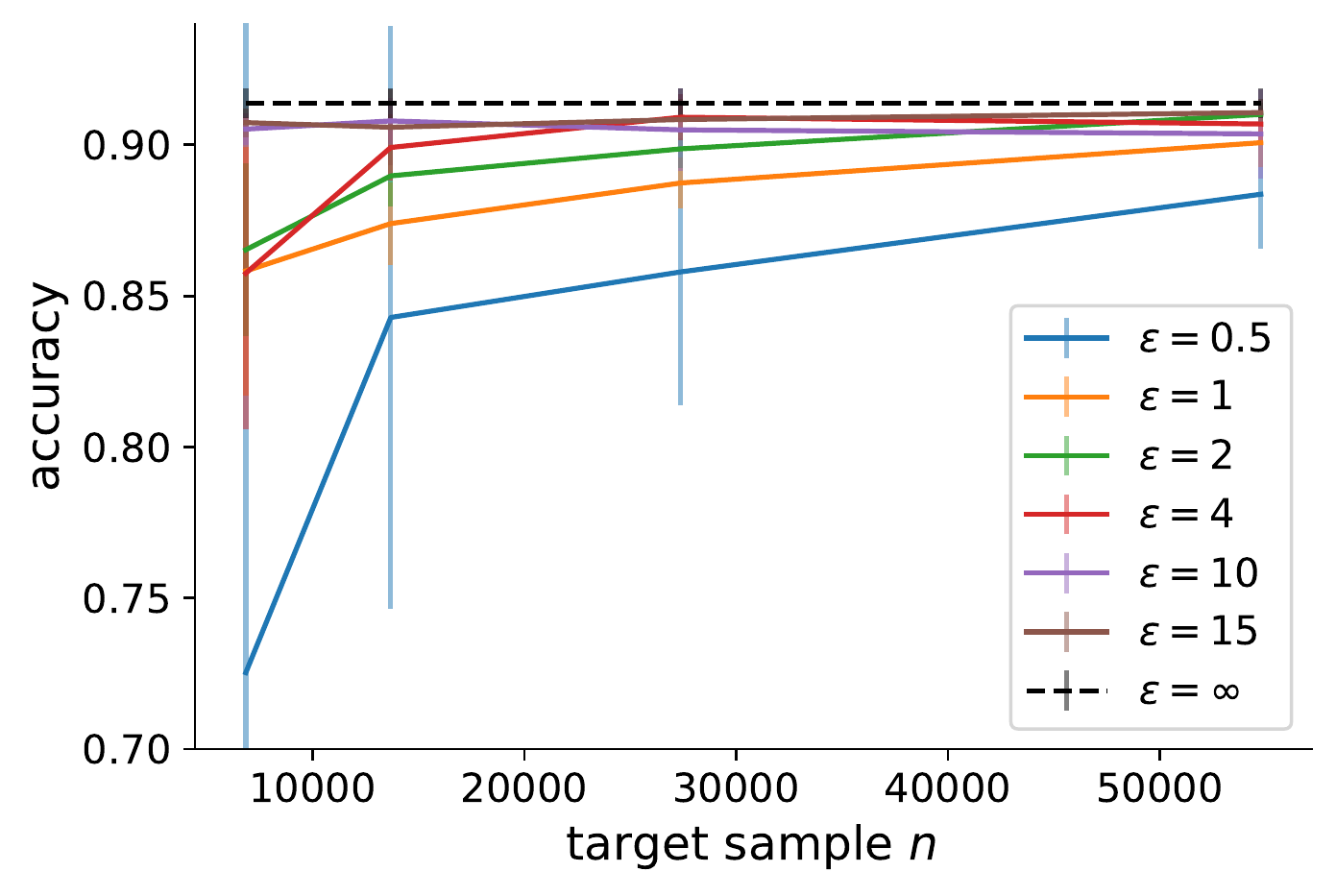}
\vskip -0.15in
\label{fig:private-classification}
\caption{Accuracy on \texttt{{Adult}} against 
  target sample size for various values of $\epsilon$.}
\end{center}
\vskip -0.3in
\end{figure} 
}

\textbf{Comparison of private and non-private
  algorithm. Classification.} We then study the performance and
convergence properties of $\pncnvx$
as a function of the privacy guarantee $\epsilon$ and training sample
size $n$.  Figure~\ref{fig:private} presents the
accuracy on the \texttt{{Adult}}, the \texttt{{CIFAR-100}}, and the \texttt{{CIFAR-10}} datasets  over ten runs against the
number of target samples with various values of $\e$. 
The dashed line corresponds to the non-private  solution of
$\ncnvx_\infty$. We observe that the performance of our private algorithm
 approaches that of the non-private solution ($\e = \infty$),
as $\e$ increases and as $n$ increases. For $\e = 10$ or $\e = 15$,
the performance is close to the dashed line for $n =
10\mathord,000$.

\section{Conclusion}

We presented two $(\e, \delta)$-differentially private algorithms for
supervised adaptation based on strong theoretical learning
guarantees. Our experimental results suggest that these algorithms can
be effective in applications and scenarios where domain adaptation can
be successful. Our proof and algorithmic techniques, such as our
reparameterization, are likely to be useful in the analysis of other
related problems (see Appendix~\ref{app:extensions}) including for the
reverse problem of adaptation from a private source to a public target
domain.

\acks{
RB's research is supported by NSF CAREER Award 2144532, NSF Award 2112471, and Google Faculty Research Award.}

\newpage
\bibliography{dpbeff}

\newpage
\appendix

\renewcommand{\contentsname}{Contents of Appendix}
\tableofcontents
\addtocontents{toc}{\protect\setcounter{tocdepth}{4}} 
\clearpage


\section{Related work}
\label{app:relatedwork}

There is a very broad literature on domain
adaptation that we cannot survey in detail within this limited
space. Thus, we refer the reader to surveys such as
\citep{pan2009,wang2018deep, li2012literature} for a relatively
comprehensive overview and briefly discuss approaches that are the
most relevant to our study.

Our analysis admits a strong theoretical component since we seek a
differentially private algorithm with theoretical learning and privacy
guarantees.  We benefit from several past publications already
referenced that have given a theoretical analysis of adaptation using
the notion of discrepancy. There are several other related
publications using the notion of discrepancy for a PAC-Bayesian
analysis \citep{GermainHabrardLavioletteMorvant2013} or active
learning \citep{DeMathelinMougeotVayatis2021}.  There are also other
interesting theoretical analyses of adaptation such as
\citep{HannekeKpotufe2019}, which deals with the notions of super
transfer or localization; these notions admit some connections with
that of (local) discrepancy
\citep{CortesMohriMunozMedina2019,ZhangLongWangJordan2020}.

In the privacy literature, several interesting algorithms have been
given with formal differentially private learning guarantees, assuming
access to public data \citep{chaudhuri2011sample,beimel2013private,
  bassily2018model,ABM19,nandi2020privately,bassily2020private}.  But
these results cannot be used in the adaptation scenario we consider
since they assume that the source and target domains coincide. A
differentially private correlation alignment approach for domain
adaptation was given by \citet{jin2021differentially} for a distinct
scenario where both source and target data are private. More recently,
\citet{wang2020deep} described algorithms for deep domain adaptation
for classification, but the authors do not provide theoretical
guarantees for these algorithms.

The problem of private density estimation using a small amounto public
data has been studied in several recent
publications. \citet{BieKamathSinghal2022} studied the problem of
estimating a $d$-dimensional Gaussian distribution, under the
assumption of access to a Gaussian that may have vanishing similarity
in total variation distance with the underlying Gaussian of the
private data.  \citet{BenDavidBieCanonneKamathSinghal2023} studied the
problem of private distribution learning with access to public data.
They related private density estimation to sample compression schemes
for distributions.  They approximately recovered previous results on
Gaussians, and presented other results such as sample complexity upper
bounds for arbitrary $k$-mixtures of Gaussians.
\citet{TramerKamathCarlini2023} presented a general discussion of the
question of private learning with large-scale public pretraining.

As discussed in the main text, the most closely related work to ours
is the recent study of \citet{BassilyMohriSuresh2020}, which considers
a similar adaptation scenario with a public source domain and a
private target domain and which also gives private algorithms with
theoretical guarantees.
However, that work can be distinguished from ours in several aspects.
First, the authors consider a purely unsupervised adaptation scenario
where no labeled sample is available from the target domain, while we
consider a supervised scenario. Our study and algorithms can be
extended to the unsupervised or weakly supervised setting using the
notion of \emph{unlabeled discrepancy}
\citep{MansourMohriRostamizadeh2009}, by leveraging upper bounds
on \emph{labeled discrepancy} in terms of unlabeled discrepancy
as in \citep{AwasthiCortesMohri2024}.
Second, the learning guarantees of our private algorithms benefit from
the recent optimization of \citet{AwasthiCortesMohri2024}, which they
show are theoretically stronger than those of the DM solution of
\citet{CortesMohri2014} adopted by \citet{BassilyMohriSuresh2020}.
Similarly, in our experiments, our convex optimization solution
outperforms the DM algorithm. Note that the empirical study in
\citep{BassilyMohriSuresh2020} is limited to a single specific
artificial dataset, while we present empirical results with several
non-artificial datasets.
Third, our private adaptation algorithms include solutions both for
regression and classification, while those of
\citet{BassilyMohriSuresh2020} are specifically given for regression
with the squared loss.

\section{General analysis of supervised adaptation}
\label{app:learningbound}

In this section, we describe the general learning bound of
\citet{AwasthiCortesMohri2024}, for which we give a self-contained and
concise proof. This bound holds for any sample reweighting method in
domain adaptation.  This includes as special cases a number of methods
presented for adaptation in the past, including KMM
\citep{HuangSmolaGrettonBorgwardtScholkopf2006}, KLIEP
\citep{SugiyamaEtAl2007a}, importance weighting with bounded weights
\citep{CortesMansourMohri2010}, discrepancy minimization
\citep{CortesMohri2014}, gapBoost algorithm
\citep{WangMendezCaiEaton2019}, and many others. Next, we discuss
the implications of this bound and the related optimization
problem.

\subsection{General learning bound}

The learning bound draws on a natural extension of the notion of
Rademacher complexity to the weighted case, \emph{$\sfq$-weighted
Rademacher complexity}, which is denoted by $\Rad_{\sfq}(\ell \circ
\sH)$ and defined by
\begin{equation}
\label{eq:q-weighted-Rad}
\Rad_{\sfq}(\ell \circ \sH) 
= \E_{S^\pub, S^\prv, \bsigma} \bracket*{\sup_{h \in \sH}
\sum_{i = 1}^{m + n} \sigma_i \sfq_i \ell(h(x_i), y_i)},
\end{equation}
where $\sigma_i$s are independent random variables 
uniformly distributed over $\mathset{-1, +1}$.
The bound holds uniformly over both the choice of a hypothesis $h$
selected in $\sH$ and that of a weight vector $\sfq$ in the open
$\norm{\, \cdot \,}_1$-ball of radius one centered in $\sfp^0$,
$\sfB_1(\sfp^0, 1) = \curl*{\sfq \colon \| \sfq - \sfp^0 \|_1 < 1}$,
where $\sfp^0$ can be interpreted as a \emph{reference} or
\emph{ideal} reweighting choice.

\begin{restatable}{theorem}{LearningBound}
\label{th:learningbound}
For any $\eta > 0$, with probability at least $1 - \eta$ over the draw
of a sample $S^\pub$ of size $m$ from $\sD^\pub$ and a sample $S^\prv$
of size $n$ from $\sD^\prv$, the following holds for all $h \in \sH$
and $\sfq \in \sfB_1(\sfp^0, 1)$:
\ifdim\columnwidth=\textwidth
{
\begin{align*}
\cL(\sD^\prv, h)
& \leq \sum_{i = 1}^{m + n} \sfq_i \ell(h(x_i), y_i)
+ \ov \sfq^\pub \dis\paren*{\sD^\prv, \sD^\pub}
+ 2 \Rad_{\sfq}(\ell \circ \sH)
+ 6B \norm{\sfq - \sfp^0}_1 \\
& \quad + B \bracket*{\norm{\sfq}_2 + 2 \norm{\sfq - \sfp^0}_1}
\bracket*{ \sqrt{\log \log_2 \tfrac{2}{1 - \norm{\sfq - \sfp^0}_1}}
  + \sqrt{\tfrac{\log (2/\eta)}{2}} }.
\end{align*}
}
\else
{
\begin{align*}
& \cL(\sD^\prv, h)\\
& \leq \sum_{i = 1}^{m + n} \sfq_i \ell(h(x_i), y_i)
+ \ov \sfq^\pub \dis\paren*{\sD^\prv, \sD^\pub}\\
& \quad + 2 \Rad_{\sfq}(\ell \circ \sH)
+ 6B \norm{\sfq - \sfp^0}_1 \\
& \quad + B \bracket*{\norm{\sfq}_2 + 2\norm{\sfq - \sfp^0}_1}
\bracket*{\! \sqrt{\log \log_2 \tfrac{2}{1 - \norm{\sfq - \sfp^0}_1}}
  \!+\! \sqrt{\! \tfrac{\log \frac{2}{\eta}}{2}} }.
\end{align*}
}
\fi
\end{restatable}

\begin{proof}
The proof of theorem consists of first deriving a $\sfq$-weighted
Rademacher complexity bound for a fixed reweighting $\sfq$, using the
fact that the expectation of the empirical term is then \[\paren*{\ov
  \sfq^\pub \cL(\sD^\pub, h) + \ov \sfq^\prv \cL(\sD^\prv, h)}.\] Next,
the difference of $\sfq^\prv \cL(\sD^\prv, h)$ and this term is
analyzed in terms of the discrepancy term $\ov \sfq^\pub
\dis\paren*{\sD^\prv, \sD^\pub}$ and then the bound is extended to
hold uniformly over $\sfB_1(\sfp^0, 1)$, using a technique similar to
that of deriving uniform margin bounds, see for example
\citep{MohriRostamizadehTalwalkar2018}[Chapter~5].
We will use $S$ to refer to the full sample: $S = (S^\pub, S^\prv)$
and will use the shorthand $\cL_S(\sfq, h) = \sum_{i = 1}^{m + n}
\sfq_i \ell(h(x_i), y_i)$ for the empirical term.

Fix $\sfq \in [0, 1]^{m + n}$. The expectation of $\cL_S(\sfq, h)$
over the draw of $(S^\pub, S^\prv) \sim (\sD^\pub)^m \times
(\sD^\pub)^n$ is then given by
  \begin{align*}
  \E_S[\cL_S(\sfq, h)]
  & = \sum_{i = 1}^{m} \sfq_i \E_{S^\pub}[\ell(h(x_i), y_i)]
  + \sum_{i = m + 1}^{m + n} \sfq_i \E_{S^\prv}[\ell(h(x_i), y_i)]\\
  & = \sum_{i = 1}^{m} \sfq_i \cL(\sD^\pub, h)
  + \sum_{i = m + 1}^{m + n} \sfq_i \cL(\sD^\prv, h)\\
  & = \ov \sfq^\pub \cL(\sD^\pub, h) + \ov \sfq^\prv \cL(\sD^\prv, h).
  \end{align*}
  Consider $\Psi(S) = \sup_{h \in \sH} \curl*{\ov \sfq^\pub
    \cL(\sD^\pub, h) + \ov \sfq^\prv \cL(\sD^\prv, h) - \cL_S(\sfq,
    h)}$. Changing point $x_i$ to $x'_i$ affects $\Psi(S)$ at most by
  $\sfq_i B$, since the loss is bounded by $B$. It is also not hard to
  see that the standard symmetrization argument (see
  \citep{MohriRostamizadehTalwalkar2018}) can be extended to the
  weighted case and that $\E_S[\Psi(S)] \le 2 \Rad_\sfq(\ell \circ
  \sH)$. Thus, by McDiarmid's inequality, for any $\eta > 0$, with
  probability at least $1 - \eta$, the following holds:
  \begin{equation}
    \label{eq:mcdiarmid}
  \ov \sfq^\pub \cL(\sD^\pub, h) + \ov \sfq^\prv \cL(\sD^\prv, h) - \cL_S(\sfq, h)
  \leq 2 \Rad_\sfq(\ell \circ \sH) + B \sqrt{\frac{\log \frac{1}{\eta}}{2m}}.
  \end{equation}
  Now, we can also analyze the difference of $\cL(\sD^\prv, h)$ and
  $\ov \sfq^\pub \cL(\sD^\pub, h) + \ov \sfq^\prv \cL(\sD^\prv, h)$ as
  follows:
  \begin{align}
    & \cL(\sD^\prv, h) - \paren*{\ov \sfq^\pub \cL(\sD^\pub, h) + \ov \sfq^\prv \cL(\sD^\prv, h)} \nonumber\\
    & = (1 - \ov \sfq^\prv) \cL(\sD^\prv, h) - \ov \sfq^\pub \cL(\sD^\pub, h) \nonumber\\
    & = (1 - \ov \sfq^\prv - \ov \sfq^\pub) \cL(\sD^\prv, h) + \ov \sfq^\pub ( \cL(\sD^\prv, h) - \cL(\sD^\pub, h)) \nonumber\\
    & = (\norm{\sfp^0}_1 - \norm{\sfq}_1) \cL(\sD^\prv, h) + \ov \sfq^\pub ( \cL(\sD^\prv, h) - \cL(\sD^\pub, h)) \nonumber\\
    \label{eq:lossinequality}
    & \leq B \norm{\sfp^0 - \sfq}_1 + \ov \sfq^\pub \dis(\sD^\prv, \sD^\pub).
  \end{align}
Combining \eqref{eq:mcdiarmid} and \eqref{eq:lossinequality} yields the following high-probability
inequality for a fixed $\sfq$:
  \begin{equation}
    \label{eq:fixedq}
    \cL(\sD^\prv, h) \leq
    \cL_S(\sfq, h) +
  2 B \norm{\sfp^0 - \sfq}_1 + \ov \sfq^\pub \dis(\sD^\prv, \sD^\pub) + \Rad_\sfq(\ell \circ \sH) + B \sqrt{\frac{\log \frac{1}{\eta}}{2m}}.
  \end{equation}
  Consider a sequence of weight vectors $\sfq_k \in [0, 1]^{m + n}$
  and a sequence of confidence weights $\eta_k = \frac{\eta}{2^{k
      + 1}}$. Inequality \eqref{eq:fixedq}, with $\sfq$ replaced by
  $\sfq^k$ and $\eta$ replaced by $\eta_k$, holds for each $k \geq
  0$, with probability $1 - \eta_k$. Thus, by the union bound, since
  $\sum_{k = 0}^{+\infty} \frac{\eta}{2^{k + 1}} = \eta$, with probability
  $1 - \eta$, it holds for all $k \geq 0$.
  \begin{equation}
  \label{eq:unifk}
  \cL(\sD^\prv, h)  \leq 
  \cL_S(\sfq^k, h) +
  2 B \norm{\sfp^0 - \sfq^k}_1 + \ov \sfq^{k, \pub} \dis(\sD^\prv, \sD^\pub)
  + \Rad_{\sfq^k}(\ell \circ \sH) + B \sqrt{\frac{\log \frac{1}{\eta}}{2m}}.
  \end{equation}
We can choose $\sfq^k$ such that $\| \sfq^k - \sfp^0 \|_1 = 1 -
\frac{1}{2^k}$.  Then, for any $\sfq \in \sfB(\sfp^0, 1)$, there
exists $k \geq 0$ such that $\| \sfq^k - \sfp^0 \|_1 \leq \| \sfq -
\sfp^0 \|_1 < \| \sfq^{k + 1} - \sfp^0 \|_1$ and thus such that
\begin{align*}
\sqrt{2 \log (k + 1)} 
= \sqrt{2 \log \log_2 \frac{2}{1 - \norm{\sfq^k - \sfp^0}_1}}
& \leq \sqrt{2 \log \log_2 \frac{2}{1 - \norm{\sfq - \sfp^0}_1}}.
\end{align*}
Furthermore, for that $k$, the following inequalities hold:
\begin{align*}
\cL_S(\sfq^k, h) 
& \leq \sum_{i = 1}^{m + n} \sfq_i \ell(h(x_i), y_i) + B \norm{\sfq^k - \sfq}_1
\leq \cL_S(\sfq, h) + 2 B \norm{\sfp^0 - \sfq}_1\\
\ov \sfq^{k, \pub}
& \leq \ov \sfq^\pub + \norm{\sfq^k - \sfq}_1
\leq \ov \sfq^\pub + 2 \norm{\sfp^0 - \sfq}_1\\
\Rad_{\sfq^k}(\ell \circ \sH)
  & \leq \Rad_{\sfq}(\ell \circ \sH) + B \norm{\sfq^k -  \sfq}_1
  \leq \Rad_{\sfq}(\ell \circ \sH) + 2 B \norm{\sfq - \sfp^0}_1,\\
\norm{\sfq^k}_2 
 & \leq \norm{\sfq}_2 + \norm{\sfq^k - \sfq}_2
 \leq \norm{\sfq}_2 + \norm{\sfq^k - \sfq}_1
 \leq \norm{\sfq}_2 + 2 \norm{\sfq - \sfp^0}_1.
\end{align*}
Plugging in these inequalities in \eqref{eq:unifk} completes the proof.
\end{proof}

\subsection{Optimization problem}

The bound suggests the following to achieve a good generalization
error in adaptation: ensure a small $\sfq$-empirical loss (first
term), but not at the price of a too sparse weight vector, which would
result in a larger $\norm{\sfq}_2$ (fifth term); allocate a smaller
total weight to public points when the discrepancy is larger (second
term); limit the $\sfq$-weighted complexity of the hypothesis set
$\sH$ combined with the loss function $\ell$ (third term); and
ensure the closeness of $\sfq$ to the reference weight vector
$\sfp^0$ (fourth and fifth terms).\ignore{ Note that, for uniform
  weights, we have $\norm{\sfq}_2 = \frac{1}{\sqrt{m + n}}$, as in
  standard generalization bounds.}

The joint optimization problem \eqref{opt-origin} is directly based on
minimizing the right-hand side of the inequality over the choice of
both $h \in \sH$ and $\sfq \in \sfB_1(\sfp^0, 1)$.
By McDiarmid's inequality, the discrepancy term $\dis\paren*{\sD^\prv,
  \sD^\pub}$ can be replaced by its estimate $\h d$
\eqref{eq:dis_estimate} from finite sample modulo a term in
$O(\sqrt{\frac{m + n}{mn}})$.
Instead of the supremum over the full family $\sH$, one can also use a
local discrepancy \cite{CortesMohriMunozMedina2019,
  DeMathelinMougeotVayatis2021,ZhangLiuLongJordan2019,
  ZhangLongWangJordan2020} and restrict oneself to a ball around the
empirical minimizer of the private loss of radius $O(1/\sqrt{n})$.
Using Talagrand's inequality \citep{LedouxTalagrand1991} and the
straightforward observation that for any $i$, $x \mapsto \sfq_i x$ is
$\norm{\sfq}_\infty$-Lipschitz, the weighted Rademacher complexity
bound can be upper bounded by $\norm{\sfq}_\infty (m + n) \Rad_{m +
  n}(\ell \circ \sH)$, where $\Rad_{m + n}(\ell \circ \sH)$ is the
standard (unweighted) Rademacher complexity of the family of loss
functions over the hypothesis set $\sH$. For uniform weights, this
is an equality.
Thus, using the upper bounds just discussed and replacing constants
with hyperparameters, minimizing the right-hand side of the learning
bounds of Theorem~\ref{th:learningbound} can be formulated as the
joint optimization problem \eqref{opt-origin}.


\section{Proofs of Section~\ref{sec:alg_linear_predictors}}
 
\subsection{Proof of Lemma~\ref{lem:_properties}}
\label{app:lemma-convex}
 
\PropertiesOfConvexObjective*

\begin{proof}
  First observe that the following inequalities hold:
  \begin{align*}
    \norm{\nabla_w \sfF(w, \sfu)}_2
    & \leq \sum_{i = 1}^{m + n}\frac{\norm{\nabla_w \sql(w, (x_i, y_i))}_2}{\sfu_i}
    \leq G\sum_{i = 1}^{n+m}\frac{1}{\sfu_i}
    \leq G,
  \end{align*}
  where the second inequality follows from the $G$-Lipschitzness of
  the loss, and the third from the constraints on $\sfu$:
  $\frac{1}{\sfu_i} \leq \frac{\alpha}{m}, \forall i \in [m]$ and
  $\frac{1}{\sfu_i} \leq \frac{1 - \alpha}{n}, \forall i \in [m + 1,
    m + n]$.
  
  Next, note that we have
  \begin{align*}
    \abs*{\frac{\partial \sfF(w, \sfu)}{\partial \sfu^\pub_i}}
    & \leq \frac{2 B}{\paren*{\sfu^\pub_i}^2} + \kappa_1\frac{\alpha^2}{m} + \kappa_2\frac{\frac{1}{\paren*{\sfu^\pub_i}^3}}{\sqrt{\sum_{i = 1}^{m + n}\frac{1}{\paren*{\sfu^\pub_i}^2}}} 
    + \kappa_\infty \frac{1\paren*{i\in \argmin_{j\in [m + n]}\sfu_j}}{\paren*{\sfu^\pub_i}^2}\\
    & \leq \frac{2 B}{\paren*{\sfu^\pub_i}^2} + \kappa_1\frac{\alpha^2}{m} + \frac{\kappa_2}{\paren*{\sfu^\pub_i}^2} + \frac{\kappa_\infty}{\paren*{\sfu^\pub_i}^2} \\
    & \leq \frac{\alpha^2}{m^2}\paren*{2B + \kappa_1+ \kappa_2 + \kappa_\infty},
  \end{align*}
  where the first inequality follows from the fact that the loss is
  uniformly bounded by $B$, and hence $\sql(w, (x_i, y_i)) + \dprv \,
  1_{i \leq m}\leq 2B$. The remaining steps follow straightforwardly
  from the constraints on $\sfu^\pub$.  Thus, we have
  $\norm{\nabla_{\sfu^\pub}\sfF(w, \sfu^\pub)}_2 \leq
  \frac{\alpha^2}{m^{3/2}}\paren*{2B + \kappa_1+ \kappa_2 +
    \kappa_\infty}$.
  
  Similarly, we have $\abs*{\frac{\partial \sfF(w, \sfu)}{\partial
      \sfu^\prv_i}}\leq \frac{(1 - \alpha)^2}{n^2}\paren*{B +
    \kappa_1+ \kappa_2 + \kappa_\infty}$ and thus
  \[\norm{\nabla_{\sfu^\prv}\sfF(w, \sfu^\prv)}_2 \leq \frac{(1 -
    \alpha)^2}{n^{3/2}}\paren*{B + \kappa_1+ \kappa_2 +
    \kappa_\infty}.\] This proves the first item of the lemma.
  
  Second, we bound the $\ell_2$-sensitivity of $\nabla_{w} \wh
  L^\prv(w, \sfu^\prv)$ and $\nabla_{\sfu^\prv} \wh L^\prv(w,
  \sfu^\prv)$. Consider any pair of neighboring private datasets
  $S^\prv$ and $\bar{S}^\prv$. Let $(x^\prv_j, y^\prv_j)$ and
  $(\bar{x}^\prv_j, \bar{y}^\prv_j)$ be the data points by which the
  two datasets differ. To emphasize the dependence on the dataset, we
  will denote $\nabla_{w} \wh L^\prv(w, \sfu^\prv)$ with respect to
  dataset $S$ as $\nabla_{w} \wh L^\prv(w, \sfu^\prv; S)$. We can write:
  \begin{align*}
    & \mspace{-80mu}
    \norm*{\nabla_{w} \wh L^\prv(w,
      \sfu^\prv; S^\prv) - \nabla_{w} \wh L^\prv(w,
      \sfu^\prv; \bar{S}^\prv)}_2\\
    & = \frac{\norm*{\nabla_w\sql(w, (x^\prv_j, y^\prv_j)) 
        - \nabla_w\sql(w, (\bar{x}^\prv_j, \bar{y}^\prv_j))}_2}{\sfu^\prv_j}\\
    & \leq \frac{2G(1 - \alpha)}{n}.
  \end{align*}
Similarly, we can write:
  \begin{align*}
    & \mspace{-80mu}
    \norm*{\nabla_{\sfu^\prv} \wh L^\prv(w, \sfu^\prv; S^\prv)
      - \nabla_{\sfu^\prv} \wh L^\prv(w, \sfu^\prv; \bar{S}^\prv)}_2\\
    & = \abs*{\frac{\sql(w, (x^\prv_j, y^\prv_j))
        - \sql(w, (\bar{x}^\prv_j, \bar{y}^\prv_j))}{\paren*{\sfu^\prv_j}^2}}\\
    & \leq \frac{B(1 - \alpha)^2}{n^2}.
  \end{align*}
  This completes the proof.
\end{proof}

\subsection{Proof of Theorem~\ref{th:guarantees-convex-priv-alg}}
\label{app:thm-convex}
 
\ConvexAlgorithmGuarantees*

We will show more precisely the following inequality:
\begin{align*}
  & \mspace{-40mu} \ex{}{\sfF(\bar{w}, \bar{\sfu}) - \sfF(w^\ast, \sfu^\ast)}\\
  & \leq \Lambda G \sqrt{\frac{1}{T} + \frac{32(1 - \alpha)^2d\log(\frac{2}{\delta})}{n^2\varepsilon^2}}\\
    & \quad + \frac{(1 - \alpha)^2\paren*{U^\prv + 1}}{2n^{3/2}}\sqrt{\frac{\bar B^2}{T} + \frac{8B^2 \log(\frac{2}{\delta})}{\varepsilon^2}}\\
    & \quad + \frac{\alpha^2\paren*{U^\pub+1}(B + \bar B)}{2m^{3/2}\sqrt{T}},
\end{align*}
where $\sfu^\ast = \paren*{\sfu^{\pub\ast}, \sfu^{\prv\ast}}$,
$U^\prv\triangleq \norm{\sfu^\prv_0- \sfu^{\prv\ast}}_2^2$, and
$U^\pub\triangleq \norm{\sfu^\pub_0- \sfu^{\pub\ast}}_2^2$.

\begin{proof}
  First, we show the privacy guarantee. Note that for any iteration
  $t \in [T]$, the only quantities that depend on the private dataset
  are the gradient components $\nabla_w \sfF(w_t, \sfu_t)$ and
  $\nabla_{\sfu^\prv}\sfF(w_t, \sfu_t)$. From the guarantees on the
  $\ell_2$-sensitivity of these gradient components given by parts~2
  and 3 of Lemma~\ref{lem:_properties} and by the properties of the
  Gaussian mechanism of differential privacy, each iteration
  $t \in [T]$ of Algorithm~\ref{Alg:ngd} is
  $(\frac{\varepsilon}{\sqrt{T}}, \delta)$-differentially
  private. Now, using the \emph{moments accountant} technique of
  \citet{abadi2016deep} to account for the composition over the $T$
  iterations of the algorithm, which applies to Gaussian noise,
  Algorithm~\ref{Alg:ngd} is $(\varepsilon, \delta)$-differentially
  private. Note that we could have also resorted to the advanced
  composition theorem of differential privacy, but the moments
  accountant technique leads to a tighter privacy analysis; in
  particular, it allows us to save a $\sqrt{\log(1/\delta)}$-factor in
  the final $\varepsilon$ and $T$ factor in the final $\delta$ privacy
  parameters.
    
  Next, we prove the bound on the optimization error. The proof
  involves some tweaks of the the standard analysis of the
  (stochastic) projected gradient descent algorithm for convex
  objectives. In particular, our proof entails decomposing each
  gradient $\nabla \sfF$ into its three components
  $\nabla_w \sfF, \nabla_{\sfu^\pub}\sfF,$ and
  $\nabla \sfF_{\sfu^ \prv}$ to allow for introducing a different step
  size for updating each of $w, \sfu^\pub$ and $\sfu^\prv$. By a
  standard argument, we have
\begin{align}
& \mspace{-20mu}
    \ex{}{\norm{w_{t+1}-w^\ast}_2^2}\nonumber\\&\leq \ex{}{\norm{w_t-w^\ast}_2^2} - 2\eta_w\ex{}{\tri{\nabla_w\sfF(w_t, \sfu_t)+\bz_t, w_t- w^\ast}}+ \eta_w^2 \ex{}{\norm{\nabla_w\sfF(w_t, \sfu_t) +\bz_t}_2^2}\nonumber\\
    &=\ex{}{\norm{w_t-w^\ast}_2^2} - 2\eta_w\ex{}{\tri{\nabla_w\sfF(w_t, \sfu_t), w_t- w^\ast}}+ \eta_w^2 \ex{}{\norm{\nabla_w\sfF(w_t, \sfu_t) +\bz_t}_2^2}\nonumber\\
    &\leq \ex{}{\norm{w_t-w^\ast}_2^2} - 2\eta_w\ex{}{\tri{\nabla_w\sfF(w_t, \sfu_t), w_t- w^\ast}}+ \eta_w^2 \paren*{G^2+ \sigma_1^2 d},\nonumber
\end{align}
where the second step follows from the linearity of expectation and
the fact that $\bz_t$ is independent of
$\paren*{(w_0, \sfu_0), \ldots, (w_t, \sfu_t)}$ and that $\bz_t$ has
zero mean, and the last step follows from the fact that
$\norm{\nabla_w\sfF(w_t, \sfu_t)}^2_2\leq G^2$ proved in
Lemma~\ref{lem:_properties} and the fact that
$\norm{\bz_t}_2^2=\sigma_1^2 d.$ Hence,
\begin{align}
    \ex{}{\tri{\nabla_w\sfF(w_t, \sfu_t), w_t- w^\ast}}&\leq \frac{\ex{}{\norm{w_t-w^\ast}_2^2} -\ex{}{\norm{w_{t+1}-w^\ast}_2^2}}{2\eta_w} + \frac{\eta_w}{2}\paren*{G^2+ \sigma_1^2 d}\label{bound:inner-prod-w}
\end{align}

By a similar argument for $\norm{\sfu^\pub_{t+1}-\sfu^\pub_t}^2_2$ and $\norm{\sfu^\prv_{t+1}-\sfu^\prv_t}^2_2$ and using Lemma~\ref{lem:_properties}, we get 

\begin{align}
& \mspace{-20mu}
    \ex{}{\tri{\nabla_{\sfu^\pub}\sfF(w_t, \sfu_t), \sfu^\pub_t- \sfu^{\pub\ast}}}\nonumber\\
    &\leq \frac{\ex{}{\norm{\sfu^\pub_t-\sfu^{\pub\ast}}_2^2} -\ex{}{\norm{\sfu^\pub_{t+1}-\sfu^{\pub\ast}}_2^2}}{2\eta_{\sfu^\pub}} + \frac{\eta_{\sfu^\pub}}{2}\ex{}{\norm{\nabla_{\sfu^\pub}\sfF(w_t, \sfu_t)}_2^2}\nonumber\\
    &\leq \frac{\ex{}{\norm{\sfu^\pub_t-\sfu^{\pub\ast}}_2^2} -\ex{}{\norm{\sfu^\pub_{t+1}-\sfu^{\pub\ast}}_2^2}}{2\eta_{\sfu^\pub}} + \frac{\eta_{\sfu^\pub}}{2}\cdot\frac{\alpha^4(\bar{B}+B)^2}{m^3}\label{bound:inner-prod-u-pub}\\
& \mspace{-20mu}
    \ex{}{\tri{\nabla_{\sfu^\prv}\sfF(w_t, \sfu_t), \sfu^\prv_t- \sfu^{\prv\ast}}}\nonumber\\
    &\leq \frac{\ex{}{\norm{\sfu^\prv_t-\sfu^{\prv\ast}}_2^2} -\ex{}{\norm{\sfu^\prv_{t+1}-\sfu^{\prv\ast}}_2^2}}{2\eta_{\sfu^\prv}} + \frac{\eta_{\sfu^\prv}}{2}\ex{}{\norm{\nabla_{\sfu^\prv}\sfF(w_t, \sfu_t)+\bz_t'}_2^2}\nonumber\\
    &\leq \frac{\ex{}{\norm{\sfu^\prv_t-\sfu^{\prv\ast}}_2^2} -\ex{}{\norm{\sfu^\prv_{t+1}-\sfu^{\prv\ast}}_2^2}}{2\eta_{\sfu^\prv}} + \frac{\eta_{\sfu^\prv}}{2}\paren*{\frac{(1-\alpha)^4\bar{B}^2}{n^3}+\sigma_2^2 n}\label{bound:inner-prod-u-prv}
\end{align}
By convexity of $\sfF$, we have
\begin{align}
    \ex{}{\sfF(w_t, \sfu_t)- \sfF(w^\ast, \sfu^\ast)}
  &\leq \ex{}{\tri{\nabla \sfF(w_t, \sfu_t), ~(w_t, \sfu_t) - (w^\ast, \sfu^\ast)}}\\
  &= \ex{}{\tri{\nabla_w \sfF(w_t, \sfu_t), w_t  - w^\ast}} + \ex{}{\tri{\nabla_{\sfu^\pub}\sfF(w_t, \sfu_t), \sfu^\pub_t- \sfu^{\pub\ast}}} \\
  &~~ +  \ex{}{\tri{\nabla_{\sfu^\prv}\sfF(w_t, \sfu_t), \sfu^\prv_t- \sfu^{\prv\ast}}} \label{bound:convexity-def}
\end{align} 
As in the standard analysis of gradient descent for convex objectives,
we combine (\ref{bound:convexity-def}) with
(\ref{bound:inner-prod-w}), (\ref{bound:inner-prod-u-pub}), and
(\ref{bound:inner-prod-u-prv}), and use the fact that
$\sfF(\bar{w}, \bar{\sfu})-\sfF(w^\ast, \sfu^\ast)\leq \frac{1}{T}
\sum_{t=1}^T \paren*{\sfF(w_t, \sfu_t)-\sfF(w^\ast, \sfu^\ast)}$
(which follows from the convexity of $\sfF$) to arrive at

\begin{align*}
  \ex{}{\sfF(\bar{w}, \bar{\sfu})- \sfF(w^\ast, \sfu^\ast)}
  \leq & \frac{\Lambda^2}{2\eta_w T} + \frac{\eta_w}{2}(G^2+d \sigma_1^2)
  + \frac{U^\prv}{2\eta_{\sfu^\prv}T}
  + \frac{\eta_{\sfu^\prv}}{2}\paren*{\frac{(1 - \alpha)^4\bar{B}^2}{n^3}
    + n \sigma_2^2}\\
  & + \frac{U^\pub}{2\eta_{\sfu^\pub}T}
  + \frac{\eta_{\sfu^\pub}}{2}\cdot\frac{\alpha^4(B + \bar{B})^2}{m^3}
\end{align*}
Substituting with the choices of $\eta_w$, $\eta_{\sfu^\pub}$, and
$\eta_{\sfu^\prv}$ in step~\ref{step:step-sizes} of
Algorithm~\ref{Alg:ngd} yields the claimed upper bound on the
expected optimization error.
 \end{proof}

\section{Proofs of Section~\ref{sec:non-convex}}
\label{app:non-convex}
 
\subsection{Proof of Lemma~\ref{lem:properties_nonconvex}}
 \label{app:lemma-nonconvex}
 
 \PropertiesOfNonConvexObjective*

\begin{proof}
  First, the proof of item~1 is similar to that of
  Lemma~\ref{lem:_properties} with minor, straightforward differences:
  first, note that that replacing the squared loss with any
  $G$-Lipschitz loss impacts neither the bounds on the norm of the
  gradient components nor the sensitivity of the gradients with
  respect to the private dataset; second, the two different terms in
  $\sfJ$ are the $\paren*{1 - \sum_{i = 1}^{m + n}\frac{1}{\sfu_i}}$
  term and the $\mu$-softmax term
  $\frac{1}{\mu}\log\paren*{\sum_{i = 1}^{m +
      n}e^{\frac{\mu}{\sfu_i}}}$ do not affect the bounds on the
  gradient norms (a straightforward calculation of the gradients of
  these terms with respect to $\sfu^\pub$ and $\sfu^\prv$, together
  with the constraints on these variables, shows that the bounds on
  the norm of the gradient components still hold) and those two terms
  also do not have any effect on the sensitivity of the gradients with
  respect to the private dataset.
 
 Next, we show the smoothness guarantee for $\sfJ$. First, note that $\sfJ$ is twice differentiable. We can express its Hessian as 
 \begin{align*}
     H(w, \sfu) & = \begin{bmatrix}
     \nabla^2_w \sfJ & \sfK\\
      & \\
       \sfK^T & \nabla^2_\sfu \sfJ
     \end{bmatrix},
 \end{align*}
 where, for any $(w, \sfu)$,
 $\nabla^2_w \sfJ(w, \sfu) \in \Rset^{d\times d}$ is given by
 $\nabla^2_w \sfJ(w, \sfu) = \bracket*{\frac{\partial^2\sfJ}{\partial
     w_i \partial w_j}(w, \sfu) \colon (i, j)\in [d]^2}$,
 $\nabla^2_\sfu \sfJ(w, \sfu) \in \Rset^{(m + n)\times (m + n)}$ by
 $\nabla^2_\sfu \sfJ(w, \sfu) =
 \bracket*{\frac{\partial^2\sfJ}{\partial \sfu_i \partial \sfu_j}(w,
   \sfu) \colon (i, j)\in [m + n]^2}$, and
 $\sfK(w, \sfu) \in \Rset^{d\times (m + n)}$ by
 $\sfK(w, \sfu) = \bracket*{\frac{\partial^2\sfJ}{\partial
     w_i \partial \sfu_j}(w, \sfu) \colon i \in [d], j \in [m +
   n]}$. Note that the spectral norm of $H$ can be upper bounded as follows:
 \begin{align*}
 \norm{H(w, \sfu)}_2
 & = \max_{\substack{V = (V_1, V_2)\\ \norm{V}_2 \leq 1}} \norm{H(w, \sfu) V}_2\\
 & = \max_{\substack{V = (V_1, V_2)\\ \norm{V}_2 \leq 1}} \norm*{
 \begin{bmatrix}
 \nabla^2_w \sfJ (w, \sfu) V_1 + \sfK(w, \sfu) V_2\\
 \sfK^\top (w, \sfu) V_1 + \nabla^2_u \sfJ (w, \sfu) V_2 \\
 \end{bmatrix}}_2\\
 &\leq  \norm{\nabla^2_w \sfJ(w, \sfu) }_2 + \norm{\sfK(w, \sfu) }_2 + \norm{\nabla^2_\sfu \sfJ(w, \sfu) }_2.
 \end{align*}
 Thus, to prove that $\sfJ$ is $O\paren*{\beta}$-smooth, it suffices for us to show that $\norm{\nabla^2_w \sfJ}_2 + \norm{\sfK}_2 + \norm{\nabla^2_\sfu\sfJ}_2  = O\paren*{\beta}$. 
 
First, observe that the following inequalities hold:
\begin{align}
    \norm*{\nabla^2_w \sfJ}_2 &= \norm*{\sum_{i = 1}^{m + n} \nabla_w^2 \ell(w, x_i, y_i)}_2
    \leq \sum_{i = 1}^{m + n}\frac{\norm*{\nabla_w^2 \ell(w, x_i, y_i)}_2}{\sfu_i}
    \leq \beta \sum_{i = 1}^{m + n}\frac{1}{\sfu_i}
    \leq \beta, \label{bound:smoothness-w}
\end{align}
where the second inequality follows from the $\beta$-smoothness of the loss $\ell$ and the last inequality from the constraints $\sfu_i \geq \frac{m}{\alpha}$ for $i \in [m]$ and $\sfu_i \geq \frac{1-\alpha}{n}$ for $i \in [m+1, m + n]$. 

Second, we bound $\norm*{\nabla^2_\sfu \sfJ}_2$. Observe that for all $i\in [m + n]$, we have
\begin{multline*}
    \frac{\partial^2 \sfJ(w, \sfu)}{\partial \sfu_i^2}= \frac{2\paren*{\ell(w, x_i, y_i)+\dprv \, 1_{i \leq m}-\lambda_1}}{\sfu_i^3} 
    + \lambda_2\frac{3\frac{1}{\sfu_i^4}\sum_{j=1}^{m + n}\frac{1}{\sfu_j^2}-\frac{1}{\sfu_i^6}}{\paren*{\sum_{j=1}^{m + n}\frac{1}{\sfu_j^2}}^{\frac{3}{2}}} + \lambda_\infty \frac{\paren*{2+\frac{\mu}{\sfu_i}}\frac{e^{\mu/\sfu_i}}{\sfu_i^3}\sum_{j=1}^{m + n}e^{\mu/\sfu_j}+\frac{\mu}{\sfu_i^4}e^{2\mu/\sfu_i}}{\paren*{\sum_{j=1}^{m + n}e^{\mu/\sfu_j}}^2}.
\end{multline*}
Thus, since the loss $\ell$ is uniformly bounded by $B$ and
$\sfu_i\geq \frac{m}{\alpha}~\forall i\in [m]$, we can bound
$\abs*{\frac{\partial^2 \sfJ(w, \sfu)}{\partial \sfu_i^2}}$ for all
$i \in [m]$ as follows:
 \begin{align*}
     \abs*{\frac{\partial^2 \sfJ(w, \sfu)}{\partial \sfu_i^2}}&\leq \frac{2\alpha^3\abs{2B-\lambda_1}}{m^3}+3\lambda_2\frac{\frac{1}{\sfu_i^4}}{\sqrt{\sum_{j=1}^{m + n}\frac{1}{\sfu_j^2}}}+\lambda_\infty \frac{\frac{2}{\sfu_i^3}e^{\mu/\sfu_i}+\frac{\mu}{\sfu_i^4}e^{\mu/\sfu_i}\sum_{j=1}^{m + n}e^{\mu/\sfu_j}}{\paren*{\sum_{j=1}^{m + n}e^{\mu/\sfu_j}}^2}\\
     &\leq \frac{2\alpha^3\abs{2B-\lambda_1}}{m^3}+\frac{3\lambda_2}{\sfu_i^3}+\lambda_\infty\paren*{\frac{2}{\sfu_i^3}+\frac{\mu}{\sfu_i^4}}\\
     &\leq \frac{2\alpha^3\abs{2B-\lambda_1}+3\lambda_2\alpha^3+2\lambda_\infty\alpha^3}{m^3}+\frac{\mu\lambda_\infty\alpha^4}{m^4}.
 \end{align*}
 Similarly, we can show that for all $i\in [m+1, m + n],$
 \begin{align*}
     \abs*{\frac{\partial^2 \sfJ(w, \sfu)}{\partial \sfu_i^2}}&\leq \frac{2(1-\alpha)^3\abs{B-\lambda_1}+3\lambda_2(1-\alpha)^3+2\lambda_\infty(1-\alpha)^3}{n^3}+\frac{\mu\lambda_\infty(1-\alpha)^4}{n^4}.
 \end{align*}
 Moreover, for all $i, j \in [m + n]$ where $i\neq j$, 
 \begin{align*}
     \abs*{\frac{\partial^2 \sfJ(w, \sfu)}{\partial \sfu_i\partial \sfu_j}}&= \frac{\frac{\lambda_2}{\sfu_i^3\sfu_j^3}}{\paren*{\sum_{t=1}^{m + n}\frac{1}{\sfu_t^2}}^{\frac{3}{2}}}+\frac{\frac{\lambda_\infty\mu}{\sfu_i^2\sfu_j^2}e^{\mu/\sfu_i}e^{\mu/\sfu_j}}{\paren*{\sum_{t=1}^{m + n}e^{\mu/\sfu_t}}^2}\\
     &\leq \frac{\lambda_2}{\sfu_i^3}+\frac{\lambda_\infty\mu}{\sfu_i^2\sfu_j^2}\\
    &\leq \begin{cases}
    \frac{\lambda_2 \alpha^3}{m^3}+\frac{\lambda_\infty\mu\alpha^4}{m^4}, & i\in [m], j\in [m] ~(i\neq j)\\
    \frac{\lambda_2 \alpha^3}{m^3}+\frac{\lambda_\infty\mu\alpha^2(1-\alpha)^2}{m^2n^2}, & i\in [m], j\in [m+1, m + n]\\
    \frac{\lambda_2 (1-\alpha)^3}{n^3}+\frac{\lambda_\infty\mu\alpha^2(1-\alpha)^2}{m^2n^2}, & i\in [m+1, m + n], j\in [m]\\
    \frac{\lambda_2 (1-\alpha)^3}{n^3}+\frac{\lambda_\infty\mu(1-\alpha)^4}{n^4}, & i\in [m+1, m + n], j\in [m+1, m + n] ~(i\neq j)
    \end{cases}
 \end{align*}
 Letting $\norm{\cdot}_F$ denote the Frobenius norm, given all the above bounds, we can bound $\norm{\nabla^2_\sfu \sfJ(w, \sfu)}_2$ as
 \begin{align}
     \norm*{\nabla^2_\sfu \sfJ(w, \sfu)}_2&\leq \norm*{\nabla^2_\sfu \sfJ(w, \sfu)}_F = \paren*{\sum_{i = 1}^{m + n}\sum_{j=1}^{m + n}\abs*{\frac{\partial^2 \sfJ(w, \sfu)}{\partial \sfu_i\partial \sfu_j}}^2}^{\frac{1}{2}}\nonumber\\
     &\leq \frac{\lambda_2\alpha^3}{m^2}+\frac{2\alpha^3\paren*{\abs*{2B-\lambda_1}+\lambda_2\sqrt{n}+\lambda_\infty}}{m^{\frac{5}{2}}}+\lambda_\infty \mu \alpha^4\paren*{\frac{1}{m^3}+\frac{1}{m^{\frac{7}{2}}}}\nonumber\\
     &~+\frac{\lambda_2(1-\alpha)^3}{n^2}+\frac{2(1-\alpha)^3\paren*{\abs{B-\lambda_1}+\lambda_2\sqrt{m}+\lambda_\infty}}{n^{\frac{5}{2}}}+\lambda_\infty \mu (1-\alpha)^4\paren*{\frac{1}{n^3}+\frac{1}{n^{\frac{7}{2}}}}\nonumber\\
     &~+ \frac{2\lambda_\infty \mu \alpha^2(1-\alpha)^2}{m^{\frac{3}{2}}n^{\frac{3}{2}}}\nonumber\\
     &\triangleq \beta'. \label{bound:smoothness-u}
 \end{align}
 Note that $\beta' = O\paren*{\frac{\sqrt{n}}{m^{5/2}}+\frac{\sqrt{m}}{n^{5/2}}+\mu\paren*{\frac{1}{m^3}+\frac{1}{n^3}}}.$ 
 Hence, when $m=O(n^3),$ $n=O(m^3),$ and $\mu=O\paren*{(m + n)^\frac{2}{3}},$ we have $\beta'=O\paren*{\frac{1}{m}+\frac{1}{n}}$. 
 
 Finally, we bound $\norm{\sfK}_2$. Observe that for all $i\in [d]$ and $j\in [m + n]$, we have
 \begin{align*}
     \abs*{\frac{\partial^2 \sfJ(w, \sfu)}{\partial w_i\partial \sfu_j}}&=\frac{\abs*{\frac{\partial \ell(w, x_j, y_j)}{\partial w_i}}}{\sfu_j^2}.
 \end{align*}
 Hence, we have 
 \begin{align}
     \norm*{\sfK}_2&\leq \norm*{\sfK}_F=\paren*{\sum_{j=1}^{m + n}\frac{1}{\sfu_j^4}\sum_{i = 1}^d \abs*{\frac{\partial \ell(w, x_j, y_j)}{\partial w_i}}^2}^{\frac{1}{2}}\nonumber\\
     &\leq G\paren*{\sum_{j=1}^{m + n}\frac{1}{\sfu_j^4}}^{\frac{1}{2}}\nonumber\\
     &\leq G\paren*{\frac{\alpha^2}{m^{3/2}}+\frac{(1-\alpha)^2}{n^{3/2}}}. \label{bound:smoothness-w-u}
 \end{align}
 Putting inequalities (\ref{bound:smoothness-w}), (\ref{bound:smoothness-u}), and (\ref{bound:smoothness-w-u}) together, we see that 
 $\sfJ$ is $\bar{\beta}$-smooth, where $\bar{\beta}=\beta+\beta' + G\paren*{\frac{\alpha^2}{m^{3/2}}+\frac{(1-\alpha)^2}{n^{3/2}}}$. As mentioned earlier, when the conditions on $m, n,$ and $\mu$ in the lemma statement are satisfied, $\beta'=O\paren*{\frac{1}{m}+\frac{1}{n}}$. Thus, under these conditions, we have $\bar{\beta}=O(\beta)$. 
\end{proof} 
 
 
 \subsection{Formal description of Algorithm \texorpdfstring{$\pncnvx$}{pncnvx} of Section~\ref{sec:non-convex}}
 \label{app:alg-nonconvex}
 
 Next, we give the pseudocode for our private algorithm (Algorithm~\ref{Alg:nonconvex-ngd}) for general adaptation scenarios described in Section~\ref{sec:non-convex}. \\
 
 \begin{algorithm}[ht]
    \caption{$\pncnvx$ Private algorithm for general adaptation scenarios based on $\sfJ$}
    \begin{algorithmic}[1]
      \REQUIRE
      $S^\pub \in \paren*{\sX \times \sY}^m$;
      $S^\prv \in \paren*{\sX \times \sY}^n$;
      privacy parameters $(\varepsilon, \delta)$;
      hyperparameters $\lambda_1, \lambda_2, \lambda_\infty$;
      number of iterations $T$.
      \STATE Choose $(w_0, \sfu_0)$ in $\sW \times \sU$ arbitrarily.
          \STATE Set $\sigma_1: =  \frac{2s_1 \sqrt{T\log(\frac{3}{\delta})}}{{\varepsilon}},$ where $s_1: =  \frac{2(1 - \alpha)G}{n}$.\nlabel{step:nc-GD-sigma-1}
          \STATE Set $\sigma_2 : =  \frac{2s_2 \sqrt{T\log(\frac{3}{\delta})}}{{\varepsilon}},$ where $s_2: =  \frac{(1 - \alpha)^2B}{n^2}$. \nlabel{step:nc-GD-sigma-2}
          \STATE Set step size $\eta := \frac{1}{\bar{\beta}}$, where $\bar{\beta}$ is the smoothness parameter given in Lemma~\ref{lem:properties_nonconvex}.\nlabel{step:nc-step-size}
          \FOR{$t = 0$ to $T-1$}
          \STATE $w_{t+1}: =  w_t - \eta  \paren*{\nabla_{w}\sfJ(w_t, \sfu_t) + \bz_t}$, where $\bz_t\sim \cN(\mathbf{0}, \sigma_1^2\mathbb{I}_d)$. \label{step:nc-noise_add_w}
          \STATE If $\norm{w_{t+1}}_2> \Lambda$ then $w_{t+1} \gets \Lambda \frac{w_{t+1}}{\norm{w_{t+1}}_2}$.\nlabel{step:nc-project-w-convex}
          \STATE $\sfu^\pub_{t+1}: =  \sfu^\pub_t - \eta \, \nabla_{\sfu^\pub}\sfJ(w_t, \sfu_t)$.
          \STATE For every $i\in [m]$, set $\sfu^\pub_{i, t+1}\gets \max\paren*{\sfu^\pub_{i, t+1}, \frac{m}{\alpha}}$. \nlabel{step:nc-project-u_pub-convex}
          \STATE $\sfu^\prv_{t+1}: = \sfu^\prv_t-\eta \, \paren*{\nabla_{\sfu^\prv}\sfJ(w_t, \sfu_t) + \bz'_t}$, where $\bz'_t\sim \cN(\mathbf{0}, \sigma_2^2\mathbb{I}_n)$. \nlabel{step:nc-noise_add_u_prv} 
          \STATE For every $i\in [n]$, set $\sfu^\prv_{i, t+1}\gets \max\paren*{\sfu^\prv_{i, t+1}, \frac{n}{1 - \alpha}}$. \nlabel{step:nc-project-u_priv-convex}
          \ENDFOR
\STATE \textbf{return} $\paren*{w_{t^\ast}, \sfu_{t^\ast}}$, where $t^\ast$ is uniformly sampled from $[T]$.
    \end{algorithmic}
    \label{Alg:nonconvex-ngd}
    
\end{algorithm}

 \subsection{Proof of Theorem~\ref{thm:nonconvex}}
 \label{app:thm-nonconvex}
 
 \NonConvexAlgorithmGuarantees*
 
 \begin{proof}
  First, we note that the privacy guarantee follows from exactly the
  same privacy argument for Algorithm~\ref{Alg:ngd} given in the proof
  of Theorem~\ref{th:guarantees-convex-priv-alg}. This because the
  differences between our algorithm in Section~\ref{sec:non-convex}
  (Algorithm~\ref{Alg:nonconvex-ngd} in
  Appendix~\ref{app:alg-nonconvex}) and Algorithm~\ref{Alg:ngd} do not
  impact the privacy analysis.
  
  We now turn to the proof of convergence to a stationary point of
  $\sfJ$ over $\sW\times \sU$ by showing that the expected norm of the
  gradient mapping of $\sfJ$ at the output $\paren*{w_{t^\ast},
    \sfu_{t^\ast}}$ is bounded as given in the theorem statement.
  
  To simplify notation, we let $\sfv_t\triangleq (w_t, \sfu_t)~
  \forall t\in [T]$, let $\sV\triangleq \sW\times \sU$, and let
  $\bg_t\triangleq \paren*{\bz_t, \mathbf{0}^m, \bz'_t}$ be the
  combined noise vector added to $\nabla\sfJ = \paren*{\nabla_w
    \sfJ(\sfv_t), \nabla_{\sfu^\pub}\sfJ(\sfv_t),
    \nabla_{\sfu^\prv}\sfJ(\sfv_t)}$ in the $t$-th iteration $\forall
  t\in [T]$. Here, $\mathbf{0}^m$ denote the $m$-dimensional all-zero
  vector. Recall that $\bz_t\sim \cN(\mathbf{0}^d,
  \sigma_1^2\mathbb{I}_d)$ and $\bz'_t \sim \cN(\mathbf{0}^n,
  \sigma_2^2\mathbb{I}_n)$ as defined in
  steps~\ref{step:nc-noise_add_w} and \ref{step:nc-noise_add_u_prv} in
  Algorithm~\ref{Alg:nonconvex-ngd}. We let
  $\bar{\sigma}^2=d\sigma_1^2+n\sigma_2^2$. Also, we let $\wt \nabla_t
  \triangleq \nabla \sfJ(\sfv_t) + \bg_t,~ \forall t\in [T]$. For any
  $v\in\Rset^{d+m + n},$ we let $\proj_{\sV}(v)$ denote the Euclidean
  projection of $v$ onto $\sV$.
  
  By $\bar{\beta}$-smoothness of $\sfJ$, we have 
  \begin{align}
      \sfJ(\sfv_{t+1})&\leq \sfJ(\sfv_t)+\tri{\nabla \sfJ(\sfv_t) , \sfv_{t+1} - \sfv_t} +\frac{\bar{\beta}}{2}\norm*{\sfv_{t+1}-\sfv_t}_2^2\nonumber\\
      &= \sfJ(\sfv_t) +\tri{\wt \nabla_t, \sfv_{t+1}-\sfv_t} - \tri{\bg_t, \sfv_{t+1}-\sfv_t} +\frac{\bar{\beta}}{2}\norm*{\sfv_{t+1}-\sfv_t}_2^2\label{ineq:smoothness}
  \end{align}
  Note that $\sfv_{t+1}=\proj_{\sV}\paren*{\sfv_t-\eta \wt\nabla_t}$. By a known property of Euclidean projection (e.g., see \cite{beck2017first}[Theorem~9.8]), we have
  \begin{align*}
      \tri{\sfv_t - \eta\wt\nabla_t - \sfv_{t+1},\sfv_t-\sfv_{t+1}}&\leq 0,
  \end{align*}
  which implies 
  \begin{align*}
      \tri{\wt\nabla_t,\sfv_t-\sfv_{t+1}}&\leq -\frac{1}{\eta}\norm{\sfv_{t+1}-\sfv_t}_2^2.
      \end{align*}
      Hence, inequality (\ref{ineq:smoothness}) implies 
      \begin{align*}
      \sfJ(\sfv_{t+1})&\leq \sfJ(\sfv_t) -\frac{1}{\eta}\paren*{1-\frac{\bar{\beta}\eta}{2}}\norm{\sfv_{t+1}-\sfv_t}_2^2 - \tri{\bg_t, \sfv_{t+1}-\sfv_t}\\
      &\leq \sfJ(\sfv_t) -\frac{1}{\eta}\paren*{1-\frac{\bar{\beta}\eta}{2}}\norm{\sfv_{t+1}-\sfv_t}_2^2 + \norm{\bg_t}_2 \norm{\sfv_{t+1}-\sfv_t}_2.
      \end{align*}
By setting $\eta=\frac{1}{\bar{\beta}}$, taking the expectation of both sides of the inequality above, and use the fact that $\ex{}{\norm{\bg_t}_2 \norm{\sfv_{t+1}-\sfv_t}_2}\leq \sqrt{\ex{}{\norm{\bg_t}_2^2}\ex{}{\norm{\sfv_{t+1}-\sfv_t}_2^2}}$, we get 
\begin{align*}
    \ex{}{\sfJ(\sfv_{t+1})}&\leq \ex{}{\sfJ(\sfv_t)}-\frac{\bar{\beta}}{2}\paren*{\sqrt{\ex{}{\norm{\sfv_{t+1}-\sfv_t}_2^2}}-\frac{\bar{\sigma}}{\bar{\beta}}}^2+\frac{\bar{\sigma}^2}{2\bar{\beta}},
\end{align*}
which implies 
\begin{align*}
     \paren*{\sqrt{\ex{}{\bar{\beta}^2\norm{\sfv_{t+1}-\sfv_t}_2^2}}-\bar{\sigma}}^2&\leq 2\bar{\beta}\ex{}{\sfJ(\sfv_{t+1})-\sfJ(\sfv_t)}+\bar{\sigma}^2.
\end{align*}
Since $\ex{}{\bar{\beta}^2\norm{\sfv_{t+1}-\sfv_t}_2^2}\leq 2\paren*{\sqrt{\ex{}{\bar{\beta}^2\norm{\sfv_{t+1}-\sfv_t}_2^2}}-\bar{\sigma}}^2+2\bar{\sigma}^2$, we get
\begin{align}
    \ex{}{\bar{\beta}^2\norm{\sfv_{t+1}-\sfv_t}_2^2}&\leq 4\bar{\beta}\ex{}{\sfJ(\sfv_{t+1})-\sfJ(\sfv_t)}+4\bar{\sigma}^2\label{ineq:noisy-grad-map-per-step}.
\end{align}
For any $t\in [T]$, let $\sfv_t^\dag\triangleq \proj_\sV \paren*{\sfv_t-\eta \nabla\sfJ(\sfv_t)}$. Observe that 
\begin{align*}
    \norm{\sfv_t - \sfv^\dag_t}_2&\leq \norm{\sfv_t-\sfv_{t+1}}_2+\norm{\sfv_{t+1}-\sfv^\dag_t}_2\\
    &=\norm{\sfv_t-\sfv_{t+1}}_2+\norm{\proj_\sV\paren*{\sfv_t-\eta \wt\nabla_t}-\proj_\sV \paren*{\sfv_t-\eta \nabla\sfJ(\sfv_t)}}_2\\
    &\leq \norm{\sfv_{t+1}-\sfv_t}_2+ \eta \norm{\bg_t}_2\\
    &= \norm{\sfv_{t+1}-\sfv_t}_2+ \frac{\norm{\bg_t}_2}{\bar{\beta}},
\end{align*}
where the first bound follows from the triangle inequality and the third bound follows from the non-expansiveness of the Euclidean projection. Thus, we have 
\begin{align*}
    \ex{}{\norm{\sfv_t - \sfv^\dag_t}_2^2}&\leq 2\ex{}{\norm{\sfv_{t+1}-\sfv_t}_2^2}+2\frac{\bar{\sigma}^2}{\bar{\beta}^2}.
\end{align*}
Combining this with (\ref{ineq:noisy-grad-map-per-step}) yields 
\begin{align*}
    \ex{}{\norm{\sG_{\sfJ, \bar{\beta}}(\sfv_t)}_2^2}=\ex{}{\bar{\beta}^2\norm{\sfv_t - \sfv^\dag_t}_2^2}&\leq 8\bar{\beta}\ex{}{\sfJ(\sfv_{t+1})-\sfJ(\sfv_t)}+10\bar{\sigma}^2.
\end{align*}
Now, taking expectation with respect to the randomness in the
uniformly drawn index $t^\ast$ of the output, we get
\begin{align*}
    \ex{}{\norm{\sG_{\sfJ, \bar{\beta}}(\sfv_{t^\ast})}_2^2}
    &=\frac{1}{T}\sum_{t=1}^T\ex{}{\norm{\sG_{\sfJ, \bar{\beta}}(\sfv_t)}_2^2}\\
    &\leq \frac{8\bar{\beta}}{T}\ex{}{\sfJ(\sfv_{0})-\sfJ(\sfv_{T})}+10\bar{\sigma}^2\\
    &\leq 8\paren*{\frac{\bar{\beta} M}{T} + \frac{40(1-\alpha)^2G^2 d T\log(2/\delta)}{\varepsilon^2 n^2}+\frac{(1-\alpha)^4B^2 T\log(2/\delta)}{\varepsilon^2 n^3}},
\end{align*}
where in the second inequality, we use the fact that $\sfJ$ is uniformly bounded over $\sV$ by $M\triangleq 2B+\lambda_1+\lambda_2 \paren*{\frac{\alpha}{\sqrt{m}}+\frac{1-\alpha}{\sqrt{n}}}+\lambda_\infty\max\paren*{\frac{\alpha}{m}, \frac{1-\alpha}{n}}.$ 
By setting 
\[
T=\frac{\sqrt{\bar{\beta}M} \varepsilon n^{3/2}}{(1-\alpha)\sqrt{\log(\frac{2}{\delta})}\sqrt{40G^2 d n +(1 - \alpha)^2 B^2}}
= O\paren*{\frac{n\varepsilon}{\sqrt{d\log(\frac{1}{\delta})}}},
\]
we finally obtain
\begin{align*}
    \ex{}{\norm{\sG_{\sfJ, \bar{\beta}}(\sfv_{t^\ast})}_2^2}&\leq 16 (1-\alpha)\sqrt{\bar{\beta}M}\sqrt{\log\paren*{\frac{2}{\delta}}}\paren*{\frac{2\sqrt{10}G \sqrt{d}}{\varepsilon n}+\frac{(1-\alpha)B}{\varepsilon n^{\frac{3}{2}}}}
    = O\paren*{\frac{\sqrt{\bar{\beta}d\log\paren*{\frac{1}{\delta}}}}{\varepsilon n}},
    \end{align*}
    which completes the proof.
 \end{proof}


\section{Additional experimental results}
\label{app:add}

\subsection{Convex setting}
\label{app:add-convex}
Table~\ref{tbl:dataset} gives the sample sizes and input dimensions of the datasets
\texttt{{Wind}}, \texttt{{Airline}}, \texttt{{Gas}}, \texttt{{News}} and \texttt{{Slice}}.
For the \texttt{{Wind}} dataset \citep{HaslettRaftery1987}, the source
and target data are collected in different months of the year with the
labels being the speed of the wind. The \texttt{{Airline}} dataset
stems from \citep{Ikonomovska}. The source and target data come from a
subset of the data for the Chicago O'Haire International Airport (ORD)
in 2008 and are divided based on different hours of the day. The goal
is to predict the amount of time the flight is delayed. The
\texttt{{Gas}} dataset \citep{RodriguezLujan2014,Vergara2012,Dua:2019}
uses features from various sensor measurements to predict the
concentration level. For the \texttt{{News}} dataset
\citep{Fernandes2015,Dua:2019}, the source contains articles from
Monday to Saturday while the target contains articles from Sunday. The task is to predict the popularity of the articles. The \texttt{{Slice}} dataset \citep{misc_relative_location_of_ct_slices_on_axial_axis_206} uses features retrieved from CT images to predict the relative location of CT slices on the axial axis of the human body. The source and target data are divided based on individual patients.

\begin{table}[t]

\begin{center}
\begin{tabular}{@{\hspace{0cm}}lrrrrr@{\hspace{0cm}}}
\toprule
 & Target training & Target validation & Target test & Source training  & Input \\
Dataset & sample size &  sample size &  sample size &  sample size  &dimension\\
\hline
\small{\tt{Wind}}    & $158$ & $200$  & $200$  & $6016$   & $11$ \\
\small{\tt{Airline}} & $200$ & $300$  & $300$  & $16000$  & $11$ \\
\small{\tt{Gas}}     & $613$ & $1000$ & $2000$ & $7297$   & $133$ \\
\small{\tt{News}}    & $737$ & $1000$ & $1000$ & $3000$   & $50$ \\
\small{\tt{Slice}}    & $1077$ & $154$ & $308$ & $7803$   & $384$ \\
\bottomrule
\end{tabular}
\end{center}
\caption{Sample sizes of the datasets \texttt{{Wind}}, \texttt{{Airline}}, \texttt{{Gas}}, \texttt{{News}} and \texttt{{Slice}}.}
\label{tbl:dataset}
\end{table}

\begin{figure}[t]
\begin{center}
\hspace{-.25cm}
\includegraphics[scale=0.57]{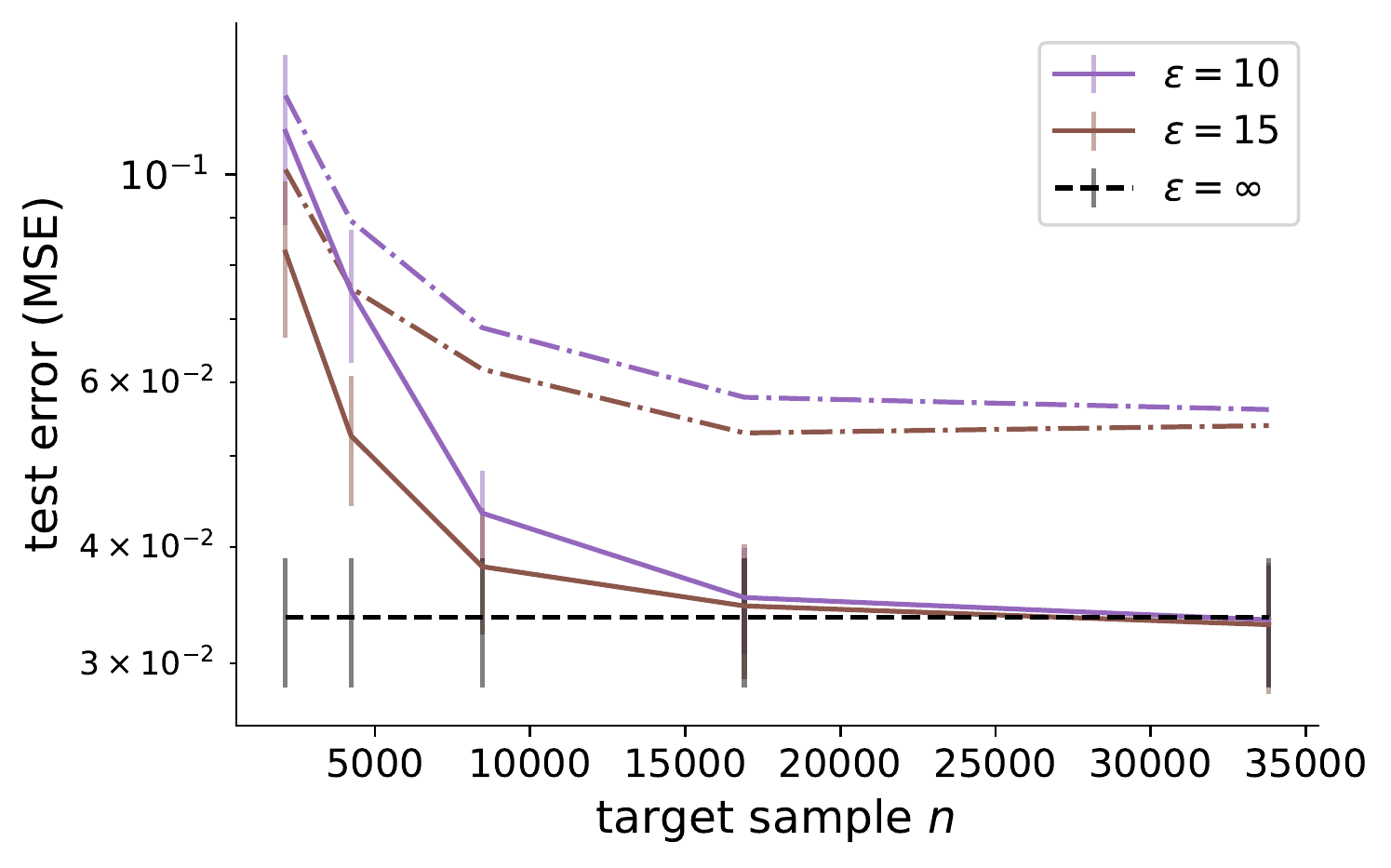}
\vskip -0.15in
\caption{MSE on the \texttt{{Gas}} dataset over ten runs against the number of target samples with various values of $\epsilon$. Comparison of our algorithm (solid lines) with the noisy minibatch
SGD from \citep{bassily2019private} (dash-dotted lines).}
\label{fig:compare}
\end{center}
\vskip -0.4in
\end{figure} 

For reference, we also compared our algorithm with the noisy minibatch
SGD from \citep{bassily2019private} (top dash-dotted plots in the figure),
see Figure~\ref{fig:compare}. We verify that our algorithm outperforms
that noisy minibatch SGD algorithm, which only benefits from the
target labeled data.

\begin{table}[t]
  \caption{MSE of private adaptation algorithm $\pcnvx$ against the non-private DM algorithm on the multi-domain sentiment analysis dataset.}
\begin{center}
\begin{tabular}{@{\hspace{0cm}}lllll@{\hspace{0cm}}}
\toprule
Source & Target &  $\pcnvx$ ($\e = 4$) &  $\pcnvx$ ($\e = 10$)&  DM\\
\hline
                     & \small{\tt{BOOKS}}& $1.649 \pm 1.909$ & $0.713 \pm 0.625$ &  $0.778 \pm 0.459$ \\
\small{\tt{KITCHEN}} & \small{\tt{DVD}}  & $0.832 \pm 0.640$ & $0.521 \pm 0.354$ &  $1.396 \pm 1.145$ \\
                     & \small{\tt{ELEC}} & $0.790 \pm 0.153$ & $0.555 \pm 0.139$ &  $1.740 \pm 1.483$ \\
\hline
                     & \small{\tt{DVD}}  & $0.929 \pm 0.365$ & $0.637 \pm 0.297$ &  $0.761 \pm 0.562$ \\
\small{\tt{BOOKS}} & \small{\tt{ELEC}}   & $0.708 \pm 0.092$ & $0.485 \pm 0.147$ &  $0.711 \pm 0.320$ \\
                  & \small{\tt{KITCHEN}} & $0.796 \pm 0.126$ & $0.632 \pm 0.181$ &  $0.956 \pm 0.442$ \\
\hline
                     & \small{\tt{ELEC}} & $0.671 \pm 0.120$ & $0.429 \pm 0.189$ &  $0.677 \pm 0.231$ \\
\small{\tt{DVD}} & \small{\tt{KITCHEN}}  & $0.665 \pm 0.197$ & $0.453 \pm 0.189$ &  $1.185 \pm 0.753$ \\
                     & \small{\tt{BOOKS}}& $0.727 \pm 0.288$ & $0.430 \pm 0.143$ &  $0.676 \pm 0.812$ \\
\hline
                   & \small{\tt{KITCHEN}}& $0.671 \pm 0.161$ & $0.439 \pm 0.141$ &  $1.389 \pm 0.755$ \\
\small{\tt{ELEC}} & \small{\tt{BOOKS}}   & $0.743 \pm 0.365$ & $0.351 \pm 0.204$ &  $0.884 \pm 0.639$ \\
                     & \small{\tt{DVD}}  & $0.755 \pm 0.585$ & $0.589 \pm 0.495$ &  $0.843 \pm 0.403$ \\
\bottomrule
\end{tabular}
\end{center}
\vskip -.2in
\label{tbl:sentiment}
\end{table}

To further illustrate the effectiveness of our algorithms, we also
report a series of additional empirical results comparing our private
adaptation algorithm $\pcnvx$ to the non-private baseline, the DM
algorithm \citep{CortesMohri2014}, on the multi-domain sentiment
analysis dataset \citep{blitzer2007biographies} formed as a regression
task for each category as in prior work
\cite{AwasthiCortesMohri2024}. We consider four categories:
\texttt{BOOKS}, \texttt{DVD}, \texttt{ELECTRONICS}, and
\texttt{KITCHEN}. We report MeanSquaredErrors, MSE, for 12 pairwise
experiments (TaskA, TaskB) in Table~\ref{tbl:sentiment}. As a source,
we use a combination of 500 examples from TaskA and 200 examples from
TaskB. For the target data we use 300 examples from TaskB. We use 50
examples from TaskB for validation and 1000 examples for testing.
The results are averaged over 10 independent source/target splits,
which show that our private adaptation algorithm $\pcnvx$ consistently
outperforms DM (non-private algorithm), even for $\e = 10$ for this
relatively small target sample size.

For a more fair comparison with private-DM \citep{BassilyMohriSuresh2020}, here, we also provide results for our private algorithm $\pcnvx$ with $\alpha =1$. In future work we seek to estimate $\alpha$ in a principled way based on the discrepancy.
Figure~\ref{fig:private_dm} shows that even for high values of $n$, the private-DM does not outperform our algorithm.

\subsection*{Hyperparameter tuning}
\label{app:hyperparameter}
While our algorithm involves hyperparameters, it is important to note that this holds for virtually all standard learning algorithms, even in the absence of adaptation; e.g., neural networks require fine-tuning of multiple parameters through validation datasets. In particular, our methodology does not rely more on hyperparameter tuning than the baselines.
Empirical hyperparameter tuning can incur a privacy cost indeed, which requires careful attention. One approach to mitigate this is using privacy-preserving hyperparameter tuning via local sensitivity analysis to estimate the privacy cost of each tuning query.

\subsection*{Sampling with replacement}
\label{app:sampling}

We use sampling with replacement to increase the number of samples in the datasets we experiment on and enable reporting results for larger values of target sample size $n$. In terms of the privacy guarantee, we would like to note that the sampling with replacement we perform does not mean that each individual in the resulting dataset contributes multiple data points. Each of the repeated data points is viewed as belonging to a different individual (that is, we assume the total number of individuals in the dataset also increases with sampling so that the size of the sampled dataset equals the total number of individuals).

\begin{figure}[t]
\begin{center}
\hspace{-.25cm}
\includegraphics[scale=0.55]{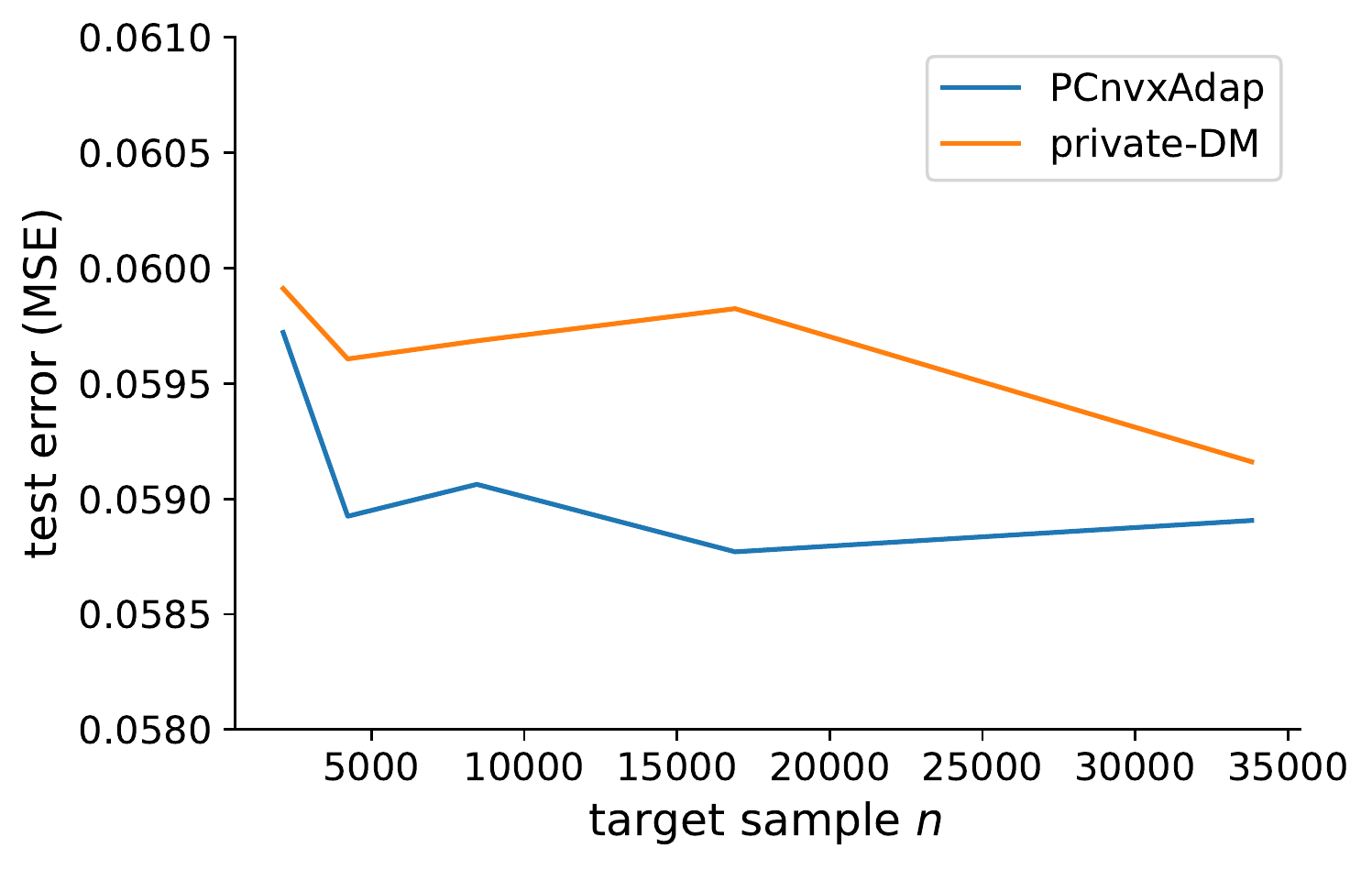}
\vskip -0.15in
\caption{Mean values of MSE on the \texttt{{Gas}} dataset over ten runs against the number of target samples with $\epsilon = 15$. Comparison of our algorithm ($\alpha = 1$) with the private-DM from \citep{BassilyMohriSuresh2020}.}
\label{fig:private_dm}
\end{center}
\vskip -0.4in
\end{figure}

\subsection{Non-convex setting}
\label{app:add-nconvex}

For the \texttt{{Adult}} dataset, also known as the Census Income
dataset, the source and target data are divided by the gender
attribute to predict whether the income exceeds $\$50$K.
The source and target data of the South German Credit dataset,
\texttt{{German}}, are divided based on whether the debtor has lived
in the present residence for at least three years. The goal is to
predict the status of the debtor's checking account with the bank.
The Speaker Accent Recognition dataset, \texttt{{Accent}}, uses
features from the soundtrack of words read by speakers from different
countries to predict the accent. The source contains examples whose
language attribute is US or UK while the target contains the remaining
examples.

We use the CLIP \citep{radford2021learning} model to extract features from the \texttt{{ImageNet}} \citep{deng2009imagenet} dataset and extract features from the \texttt{{CIFAR-100}}, \texttt{{CIFAR-10}} \citep{Krizhevsky09learningmultiple} and \texttt{{SVHN}} \citep{Netzer2011} datasets by using the outputs of the second-to-last layer of ResNet \citep{he2016deep}. 
We transform those datasets into binary classification by assigning half of the labels as $+1$ and the other half as $-1$ and then convert them into domain adaptation tasks where the source and target data consist of distinct mixtures of uniform sampling and Gaussian sampling using the mean and covariance of the data. For the \texttt{{CIFAR-100}} and \texttt{{ImageNet}} datasets, $95\%$ of the source data and $5\%$ of the target data come from Gaussian sampling; for the \texttt{{CIFAR-10}} dataset, $90\%$ of the source data and $10\%$ of the target data come from Gaussian sampling; for the \texttt{{SVHN}} dataset, $80\%$ of the source data and $20\%$ of the target data come from Gaussian sampling.

\section{Extensions}
\label{app:extensions}

\textbf{Reverse scenario.} Notably, given the form of the learning
bound and the objective function adopted in our private
optimization-based adaptation algorithms, they can be
straightforwardly modified to derive private algorithms with similar
guarantees in the reverse scenario where the source domain is private
while the target domain is public.

\noindent \textbf{Different feature spaces.} Our current formulation assumes the
same input space for the source and target distributions. For
adaptation problems with distinct input feature spaces, two approaches
can be taken to extend our results: (1) If the mapping $\Psi$ between
the spaces is known, our theory can be readily extended to accommodate
this situation. In this case, the adaptation algorithm can be adjusted
by incorporating the mapping $\Psi$ in the learning process; we
consider $(h \circ \Psi)(x)$ for source domain instances; (2) If the
mapping $\Psi$ is unknown, it needs to be learned simultaneously with
$h$. In this scenario, we consider $(h \circ \Psi)(x)$ for source
domain instances, with $\Psi$ being an integral part of the learning
process. We leave a detailed analysis, including a needed extension of
the generalization bound of Theorem~\ref{th:learningbound}, to future
work. The case of distinct output spaces is similar.

\end{document}